\documentclass[twoside,11pt]{article}

\usepackage{jmlr2e}

\usepackage{mathtools}

\usepackage{amssymb}
\usepackage{amsthm}
\usepackage{color}

\newcommand{\cut}[1]{{}}

\newcommand{\va}{{\mathbf{a}}}
\newcommand{\vb}{{\mathbf{b}}}

\newcommand{\ve}{{\mathbf{e}}}

\newcommand{\vg}{{\mathbf{g}}}
\newcommand{\vh}{{\mathbf{h}}}

\newcommand{\vq}{{\mathbf{q}}}

\newcommand{\vs}{{\mathbf{s}}}

\newcommand{\vx}{{\mathbf{x}}}
\newcommand{\vy}{{\mathbf{y}}}

\newcommand{\vA}{{\mathbf{A}}}

\newcommand{\vR}{{\mathbf{R}}}

\newcommand{\vV}{{\mathbf{V}}}

\newcommand{\cD}{{\mathcal{D}}}

\newcommand{\EE}{{\mathbb{E}}}

\newcommand{\RR}{\mathbb{R}}

\newcommand{\sign}{\mathrm{sign}}
\newcommand{\vzero}{\mathbf{0}}

\newcommand{\prox}{\mathbf{prox}}

\makeatletter
\let\@@span\span
\def\sp@n{\@@span\omit\advance\@multicnt\m@ne}
\makeatother

\DeclareMathOperator*{\Min}{minimize}

\DeclarePairedDelimiter{\dotp}{\langle}{\rangle}

\newcommand{\bc}{\begin{center}}
\newcommand{\ec}{\end{center}}

\newcommand{\bdm}{\begin{displaymath}}
\newcommand{\edm}{\end{displaymath}}

\newcommand{\beq}{\begin{equation}}
\newcommand{\eeq}{\end{equation}}

\newcommand{\bfl}{\begin{flushleft}}
\newcommand{\efl}{\end{flushleft}}

\newcommand{\bt}{\begin{tabbing}}
\newcommand{\et}{\end{tabbing}}

\newcommand{\beqn}{\begin{align}}
\newcommand{\eeqn}{\end{align}}

\newcommand{\beqs}{\begin{align*}} 
\newcommand{\eeqs}{\end{align*}}  

\newtheorem{theorem}{Theorem}
\newtheorem{assumption}{Assumption}

\newtheorem{corollary}{Corollary}
\newtheorem{remark}{Remark}
\newtheorem{lemma}{Lemma}

\usepackage{algorithm}
\usepackage{algorithmic}
\usepackage{wrapfig}
\usepackage{subfig}
\usepackage{enumitem}
\usepackage{multicol}
\usepackage{natbib}

\usepackage[affil-it]{authblk}

\title{\textbf{A Double Residual Compression Algorithm \\ for Efficient Distributed Learning}}

\author[1]{Xiaorui Liu}
\author[2]{Yao Li}
\author[1]{Jiliang Tang}
\author[3]{Ming Yan}
\affil[1]{Department of Computer Science and Engineering}
\affil[2]{Department of Mathematics}
\affil[3]{Department of Computational Mathematics, Science and Engineering} 
\affil[ ]{Michigan State University}
\affil[ ]{\{xiaorui, liyao6, tangjili, myan\}@msu.edu}
\date{}

\begin{document}
\maketitle

\graphicspath{{./figure/}}

\begin{abstract}
Large-scale machine learning models are often trained by parallel stochastic gradient descent algorithms. However, the communication cost of gradient aggregation and model synchronization between the master and worker nodes becomes the major obstacle for efficient learning as the number of workers and the dimension of the model increase. In this paper, we propose DORE, a DOuble REsidual compression stochastic gradient descent algorithm, to reduce over $95\%$ of the overall communication such that the obstacle can be immensely mitigated. Our theoretical analyses demonstrate that the proposed strategy has superior convergence properties for both strongly convex and nonconvex objective functions. 
The experimental results validate that DORE achieves the best communication efficiency while maintaining similar model accuracy and convergence speed in comparison with start-of-the-art baselines.
\end{abstract}

\section{Introduction}
Stochastic gradient algorithms~\citep{sgd2010} are efficient at minimizing the objective function $f: \mathbb{R}^d \rightarrow \mathbb{R}$ which is usually defined as $f(\vx):=\EE_{\xi\sim \cD}[\ell(\vx,\xi)]$, where $\ell(\vx,\xi)$ is the objective function defined on data sample $\xi$ and model parameter $\vx$. 
A basic stochastic gradient descent (SGD) repeats the gradient ``descent'' step
$\vx^{k+1} = \vx^{k} - \gamma \vg(\vx^{k})$
where $\vx_{k}$ is the current iteration and $\gamma$ is the step size. 
The stochastic gradient $\vg(\vx^k)$ is computed based on an i.i.d. sampled mini-batch from the distribution of the training data $\cD$ and serves as the estimator of the full gradient $\nabla f(\vx^k)$.
In the context of large-scale machine learning, the number of data samples and the model size are usually very large. 
Distributed learning utilizes a large number of computers/cores to perform the stochastic algorithms aiming at reducing the training time.
It has attracted extensive attention due to the demand for highly efficient model training~\citep{tensorflow16, mxnetpaper, Li:2014:SDM:2685048.2685095, You:2018:ITM:3225058.3225069}. 

In this paper, we focus on the data-parallel SGD~\citep{Dean:2012:LSD:2999134.2999271,NIPS2015_5751,NIPS2010_4006}, which provides a scalable solution to speed up the training process by distributing the whole data to multiple computing nodes. 
The objective can be written as:
\begin{equation*}\label{pb1} 
\textstyle\Min\limits_{\vx\in\RR^{d}}f(\vx)+R(\vx)=\frac{1}{n}\sum\limits_{i=1}^{n} \underbrace{\EE_{\xi\sim \cD_{i}}[\ell(\vx,\xi)]}_{\coloneqq f_{i}(\vx)} +R(\vx),
\end{equation*} 
where each $f_{i}(\vx)$ is a local objective function of the worker node $i$ defined based on the allocated data under distribution $\cD_{i}$ and $R:\RR^{d}\rightarrow\RR$ is usually a closed convex regularizer. 

In the well-known parameter server framework~\citep{Li:2014:SDM:2685048.2685095, NIPS2010_4006},
during each iteration, each worker node evaluates its own stochastic gradient $\{\widetilde\nabla f_{i}(\vx^k)\}_{i=1}^n$ and send it to the master node, which collects all gradients and calculates their average $(1/n)\sum_{i=1}^n \widetilde\nabla f_{i}(\vx^k)$. 
Then the master node further takes the gradient descent step with the averaged gradient and broadcasts the new model parameter $\vx^{k+1}$ to all worker nodes. 
It makes use of the computational resources from all nodes. In reality, the network bandwidth is often limited. Thus, the communication cost for the gradient transmission and model synchronization becomes the dominating bottlenecks as the number of nodes and the model size increase, which hinders the scalability and efficiency of SGD.

One common way to reduce the communication cost is to compress the gradient information by either gradient sparsification or quantization~\citep{alistarh2017qsgd, 1-bit-sgd, Stich:2018:SSM:3327345.3327357, Strom2015ScalableDD,pmlr-v70-wang17f,  Wangni:2018:GSC:3326943.3327063, wen2017terngrad, pmlr-v80-wu18d} such that many fewer bits of information are needed to be transmitted. 
However, little attention has been paid on how to reduce the communication cost for model synchronization and the corresponding theoretical guarantees. 
Obviously, the model shares the same size as the gradient, so does the communication cost. 
Thus, merely compressing the gradient can reduce at most 50\% of the communication cost, which suggests the importance of model compression.
Notably, the compression of model parameters is much more challenging than gradient compression.
One key obstacle is that its compression error cannot be well controlled by the step size $\gamma$ and thus it cannot diminish like that in the gradient compression~\citep{DBLP:conf/nips/TangGZZL18}. 
In this paper, we aim to bridge this gap by investigating algorithms to compress the full communication in the optimization process and understanding their theoretical properties. Our contributions can be summarized as:
\begin{itemize}[leftmargin=0.2in]
    \item We proposed DORE, which can compress both the gradient and the model information such that more than $95\%$ of the communication cost can be reduced.
    \item We provided theoretical analyses to guarantee the convergence of DORE under strongly convex and nonconvex assumptions without the bounded gradient assumption. 
    \item Our experiments demonstrate the superior efficiency of DORE comparing with the state-of-art baselines without degrading the convergence speed and the model accuracy.
\end{itemize}

\section{Background} 
Recently, many works try to reduce the communication cost to speed up the distributed learning, especially for deep learning applications, where the size of the model is typically very large (so is the size of the gradient) while the network bandwidth is relatively limited. Below we briefly review relevant papers.

\textbf{Gradient quantization and sparsification.}
Recent works~\citep{alistarh2017qsgd,1-bit-sgd, wen2017terngrad, mishchenko2019distributed,signSGD-ICML18} have shown that the information of the gradient can be quantized into a lower-precision vector such that fewer bits are needed in communication without loss of accuracy.~\cite{1-bit-sgd} proposed 1Bit SGD that keeps the sign of each element in the gradient only. 
It empirically works well, and~\cite{signSGD-ICML18} provided theoretical analysis systematically. QSGD~\citep{alistarh2017qsgd} utilizes an unbiased multi-level random quantization to compress the gradient while Terngrad~\citep{wen2017terngrad} quantizes the gradient into ternary numbers $\{0,\pm1\}$.
In DIANA~\citep{mishchenko2019distributed}, the gradient difference is compressed and communicated contributing to the estimator of the gradient in the master node.

Another effective strategy to reduce the communication cost is sparsification.
\cite{Wangni:2018:GSC:3326943.3327063} proposed a convex optimization formulation to minimize the coding length of stochastic gradients. A more aggressive sparsification method is to keep the elements with relatively larger magnitude in gradients, such as top-k sparsification ~\citep{Stich:2018:SSM:3327345.3327357, Strom2015ScalableDD, aji-heafield-2017-sparse}.

\textbf{Model synchronization.} 
The typical way for model synchronization is to broadcast model parameters to all worker nodes. 
Some works~\citep{pmlr-v70-wang17f,jordanm} have been proposed to reduce model size by enforcing sparsity, but it cannot be applied to general optimization problems. 
Some alternatives including QSGD~\citep{alistarh2017qsgd} and ECQ-SGD~\citep{pmlr-v80-wu18d} choose to broadcast all quantized gradients to all other workers such that every worker can perform model update independently. However, all-to-all communication is not efficient since the number of transmitted bits increases dramatically in large-scale networks. 
DoubleSqueeze~\citep{tang2019doublesqueeze} applies compression on the averaged gradient with error compensation to speed up model synchronization.

\textbf{Error compensation.} 
~\cite{1-bit-sgd} applied error compensation on 1Bit-SGD and achieved negligible loss of accuracy empirically. Recently, error compensation was further studied~\citep{pmlr-v80-wu18d, Stich:2018:SSM:3327345.3327357, KarimireddyRSJ19feedback} to mitigate the error caused by compression. The general idea is to add the compressed error to the next compression step: 
\vspace{-0.in}
$$
\hat \vg = Q(\vg+\ve), ~~~
\ve = (\vg+\ve) - \hat \vg.
$$
However, to the best of our knowledge, most of the algorithms with error compensation~\citep{pmlr-v80-wu18d, Stich:2018:SSM:3327345.3327357, KarimireddyRSJ19feedback, tang2019doublesqueeze} need to assume bounded gradient, i.e., $\EE\|\vg\|^2 \leq B$, and the convergence rate depends on this bound. 

\textbf{Contributions of DORE.}
The most related papers to DORE are DIANA~\citep{mishchenko2019distributed} and DoubleSqueeze~\citep{tang2019doublesqueeze}. Similarly, DIANA compresses gradient difference on the worker side and achieves good convergence rate. However, it doesn't consider the compression in model synchronization, so at most 50\% of the communication cost can be saved. DoubleSqueeze applies compression with error compensation on both worker and server sides, but it only considers non-convex objective functions. Moreover, its analysis relies on a bounded gradient assumption, i.e., $\EE\|\vg\|^2 \leq B$, and the convergence error has a dependency on the gradient bound like most existed error compensation works.

In general, the uniform bound on the norm of the stochastic gradient is a strong assumption which might not hold in some cases. For example, it is violated in the strongly convex case~\citep{pmlr-v80-nguyen18c, gower2019sgd}. In this paper, we design DORE, the first algorithm which utilizes gradient and model compression with error compensation without assuming bounded gradients. 
Unlike existing error compensation works, we provide a linear convergence rate to the $\mathcal{O}(\sigma)$ neighborhood of the optimal solution for strongly convex functions and a sublinear rate to the stationary point for nonconvex functions with linear speedup. In Table~\ref{table}, we compare the asymptotic convergence rates of different quantized SGDs with DORE. 

\section{Double Residual Compression SGD}
\label{sec:alg}

In this section, we introduce the proposed \underline{DO}uble \underline{RE}sidual compression SGD (DORE) algorithm. 
Before that, we introduce a common assumption for the compression operator. 

In this work, we adopt an assumption from~\citep{alistarh2017qsgd, wen2017terngrad, mishchenko2019distributed} that the compression variance is linearly proportional to the magnitude. 

\begin{assumption}\label{ass:compression}
The stochastic compression operator $Q:\RR^{d}\rightarrow\RR^{d}$ is unbiased, i.e., $\EE Q(\vx)=\vx$ and satisfies
\begin{equation}\label{quan_var}
\EE \|Q(\vx)-\vx\|^{2}\leq C\|\vx\|^{2},
\end{equation}
for a nonnegative constant $C$ that is independent of $\vx$.
We use $\hat{\vx}$ to denote the compressed $\vx$, i.e., $\hat{\vx}\sim Q(\vx)$.
\end{assumption}

Many feasible compression operators can be applied to our algorithm since our theoretical analyses are built on this common assumption. 
Some examples of feasible stochastic compression operators include:
\begin{itemize}[leftmargin=0.2in]
    \item \textit{No Compression:} $C=0$ when there is no compression. 
    \item \textit{Stochastic Quantization:} A real number $x \in [a,b], (a<b)$ is set to be $a$ with probability $\frac{b-x}{b-a}$ and $b$ with probability $\frac{x-a}{b-a}$, where $a$ and $b$ are predefined quantization levels~\citep{alistarh2017qsgd}. It satisfies Assumption~\ref{ass:compression} when $ab>0$ and $a<b$. 
    \item \textit{Stochastic Sparsification:} A real number $x$ is set to be 0 with probability $1-p$ and $\frac{x}{p}$ with probability $p$~\citep{wen2017terngrad}. It satisfies Assumption~\ref{ass:compression} with $C=(1/p)-1$.
    \item \textit{$p$-norm Quantization:} A vector $\vx$ is quantized element-wisely by $Q_{p}(\vx)=\|\vx\|_{p}~\sign(\vx)\circ \xi$, where $\circ$ is the Hadamard product and $\xi$ is a Bernoulli random vector satisfying $\xi_{i}\sim \mbox{Bernoulli}(\frac{|x_{i}|}{\|\vx\|_{p}})$. It satisfies Assumption~\ref{ass:compression} with $C=\max_{\vx\in\RR^d} \frac{\|\vx\|_1\|\vx\|_p}{\|\vx\|_2^2}-1$~\citep{mishchenko2019distributed}.
    To decrease the constant $C$ for a higher accuracy, we can further decompose a vector $\vx \in\RR^d$ into blocks, i.e., $\vx=(\vx(1)^{\top}, \vx(2)^{\top}, \cdots, \vx(m)^{\top})^{\top}$ with $\vx(l)\in\RR^{d_{l}}$ and $\sum_{l=1}^{m}d_{l}=d$, and compress the blocks independently.
\end{itemize}

\subsection{The Proposed DORE}

\begin{figure}
   \begin{center}
    \includegraphics[width=0.6\textwidth]{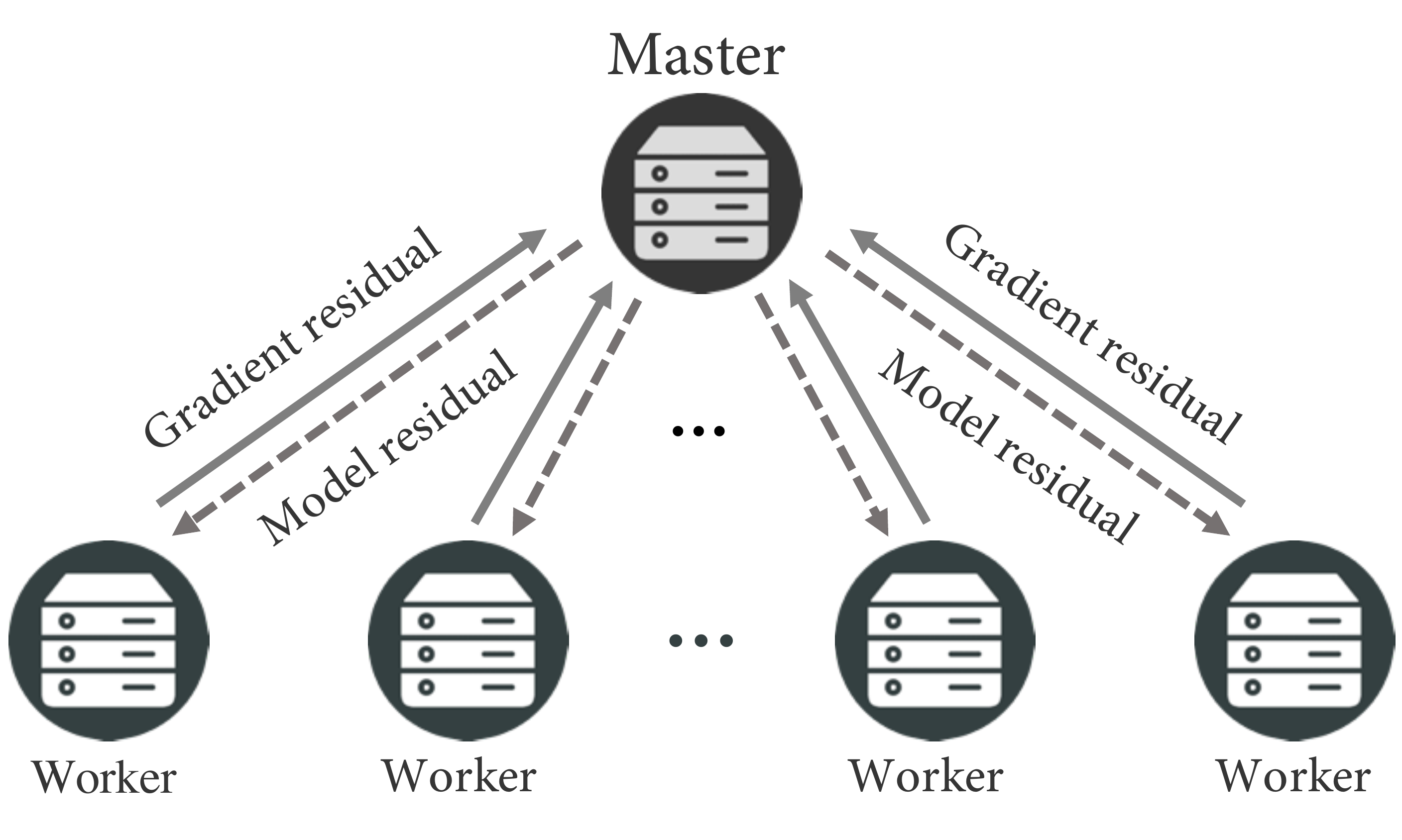}
  \end{center}
  \caption{An Illustration of DORE}
  \label{fig:ps_wk}
\end{figure}

Many previous works~\citep{alistarh2017qsgd, 1-bit-sgd,wen2017terngrad} reduce the communication cost of P-SGD by quantizing the stochastic gradient before sending it to the master node, but there are several intrinsic issues. 

First, these algorithms will incur extra optimization error intrinsically. 
Let's consider the case when the algorithm converges to the optimal point $\vx^*$ where we have $(1/n)\sum_{i=1}^n \nabla f_{i}(\vx^*)=\vzero$. 
However, the data distributions may be different for different worker nodes in general, and thus we may have $\nabla f_{i}(\vx^*) \neq \nabla f_{j}(\vx^*), \forall i, j\in \{1,\dots,n\}$ and $i\neq j$. 
In other words, each individual $\nabla f_i(\vx^*)$ may be far away from zero. This will cause large compression variance according to Assumption~\ref{ass:compression}, which indicates that the upper bound of compression variance $\EE\|Q(\vx)-\vx\|^{2}$ is linearly proportional to the magnitude of $\vx$. 

Second, most existing algorithms~\citep{1-bit-sgd, alistarh2017qsgd, wen2017terngrad, signSGD-ICML18, pmlr-v80-wu18d, mishchenko2019distributed} need to broadcast the model or gradient to all worker nodes in each iteration. It is a considerable bottleneck for efficient optimization since the amount of bits to transmit is the same as the uncompressed gradient. DoubleSqueeze~\citep{tang2019doublesqueeze} is able to apply compression on both worker and server sides. However, its analysis depends on a strong assumption on bounded gradient. Meanwhile, no theoretical guarantees are provided for the convex problems.

\begin{algorithm*}[!ht]
\caption{The Proposed DORE.\protect\footnotemark}
\label{algocomplete}
\small
\begin{algorithmic}[1]
\STATE \textbf{Input:} Stepsize $\alpha, \beta, \gamma, \eta$, initialize $\vh^0 = \vh_i^0 = \vzero^{d}$, $\hat{\vx}_i^0 = \hat{\vx}^0, ~\forall i\in\{1,\dots, n\}$. \label{input1}
\FOR{$k=1,2,\cdots, K-1$}
\vspace{0.02in}
\begin{minipage}[t]{0.48\textwidth}
    \STATE \textbf{For each worker} $i \in \{1,2,\cdots, n\}$:
    \vspace{0.05in}
    \STATE Sample $\vg_i^k$ such that $\EE [\vg_i^{k}|\hat{\vx}_i^{k}]=\nabla f_i(\hat{\vx}_i^k)$
    \STATE Gradient residual: $\Delta_i^k = \vg_i^k - \vh_i^k $\label{ite1}
    \STATE Compression: $\hat{\Delta}_i^k = Q(\Delta_i^k)$
    \STATE $\vh_i^{k+1} = \vh_i^k + \alpha \hat{\Delta}_i^k$\label{ite3}
    \STATE \{~$\hat{\vg}_i^k = \vh_i^k + \hat{\Delta}_i^k$~\}
    \STATE Send $\hat{\Delta}_i^k$ to the master 
    \STATE Receive $\hat{\vq}^k$ from the master
    \STATE $\hat{\vx}_i^{k+1} = \hat{\vx}_i^k + \beta\hat{\vq}^k$
\end{minipage}
\begin{minipage}[t]{0.42\textwidth}
\STATE \textbf{For the master}:
    \vspace{0.05in}
    \STATE Receive $\{\hat{\Delta}_{i}^{k}\}$ from workers
    \STATE $\hat{\Delta}^k = 1/n\sum_i^n {\hat{\Delta}_i^k}$
    \STATE $\hat{\vg}^k = \vh^k + \hat{\Delta}^k ~~\{=1/n \sum_i^n {\hat{\vg}_i^k}~\} $
    \STATE $\vx^{k+1} = \prox_{\gamma R}(\hat{\vx}^{k} - \gamma \hat{\vg}^k)$
    \STATE $\vh^{k+1} = \vh^k + \alpha \hat{\Delta}^k$ 
    \STATE Model residual: $\vq^k = \vx^{k+1} -\hat{\vx}^k + \eta \ve^k$
    \STATE Compression: $\hat{\vq}^k = Q(\vq^k)$
    \STATE $\ve^{k+1} = \vq^k - {\hat \vq}^k$ 
    \STATE $\hat{\vx}^{k+1}=\hat{\vx}^k+\beta \hat{\vq}^{k}$
    \STATE Broadcast $\hat{\vq}^{k}$ to workers
\end{minipage}
\vspace{-0.1in}
\ENDFOR
\STATE \textbf{Output:} $\hat{\vx}^{K}$ or any $\hat{\vx}_i^{K}$ 
\end{algorithmic}
\normalsize
\end{algorithm*}

We proposed DORE to address all aforementioned issues. Our motivation is that the gradient should change smoothly for smooth functions so that each worker node can keep a state variable $\vh_i^k$ to track its previous gradient information. 
As a result, the residual between new gradient and the state $\vh_i^k$ should decrease, and the compression variance of the residual can be well bounded. 
On the other hand, as the algorithm converges, the model would only change slightly. 
Therefore, we propose to compress the model residual such that the compression variance can be minimized and also well bounded. We also compensate the model residual compression error into next iteration to achieve a better convergence. Due to the advantages of the proposed double residual compression scheme, we can derive the fastest convergence rate through analyses without the bounded gradient assumption. 
Below are some key steps of our algorithm as showed in Algorithm~\ref{algocomplete} and Figure~\ref{fig:ps_wk}:
\begin{itemize}[leftmargin=0.2in]
    \item[] {\bf [lines 4-9]:} each worker node sends the compressed gradient residual ($\hat \Delta_i^k$) to the master node and updates its state $\vh_i^k$ with $\hat \Delta_i^k$;
    \item[] {\bf [lines 13-15]:} the master node gathers the compressed gradient residual $(\{\hat \Delta_i^k)\}$ from all worker nodes and recovers the averaged gradient $\hat \vg^k$ based on its state $\vh^k$; 
    \item[] {\bf [lines 16]:} the master node applies gradient descent algorithms (possibly with the proximal operator);
    \item[] {\bf [lines 18-22]:} the master node broadcasts the compressed model residual with error compensation ($\hat \vq^k$) to all worker nodes and updates the model;  
    \item[] {\bf [lines 10-11]:} each worker node receives the compressed model residual ($\hat \vq^k$) and updates its model $\vx_i^k$.
\end{itemize}

In the algorithm, the state $\vh_i^k$ serves as an exponential moving average of the local gradient in expectation, i.e., $\EE_{Q}\vh_{i}^{k+1}= (1-\alpha)\vh_{i}^{k}+\alpha \vg_{i}^{k}$, as proved in Lemma~\ref{lem3}. 
Therefore, as the iteration approaches the optimum, $\vh_i^k$ will also approach the local gradient $\nabla f_i(\vx^*)$ rapidly which contributes to small gradient residual and consequently small compression variance. 
Similar difference compression techniques are also proposed in DIANA and its variance-reduced variant~\citep{mishchenko2019distributed, VRdiana}. 

\footnotetext{Equations in the curly bracket are just notations for the proof but does not need to computed actually.}

\subsection{Discussion}
In this subsection, we provide more detailed discussions about DORE including model initialization, model update, the special smooth case as well as the compression rate of communication.

{\bf{Initialization.}} It is important to take the identical initialization $\hat \vx^0$ for all worker and master nodes.
It is easy to be ensured by either setting the same random seed or broadcasting the model once at the beginning.
In this way, although we don't need to broadcast the model parameters directly, every worker node updates the model $\hat \vx^k$ in the same way. Thus we can keep their model parameters identical. 
Otherwise, the model inconsistency needs to be considered. 

{\bf{Model update.}} It is worth noting that although we can choose an accurate model $\vx^{k+1}$ as the next iteration in the master node, we use $\hat \vx^{k+1}$ instead. In this way, we can ensure that the gradient descent algorithm is applied based on the exact stochastic gradient which is evaluated on $\hat \vx_i^k$ at each worker node. This dispels the intricacy to deal with inexact gradient evaluated on $\vx^{k}$ and thus it simplifies the convergence analysis. 

{\bf{Smooth case.}} In the smooth case, i.e., $R=0,$ Algorithm~\ref{algocomplete} can be simplified.  The master node quantizes the recovered averaged gradient with error compensation and broadcasts it to all worker nodes. The details of this simplified case can be found in Appendix~\ref{sec:smoothcase}.

{\bf{Compression rate.}} The compression of the gradient information can reduce at most $50\%$ of the communication cost since it only considers compression during gradient aggregation while ignoring the model synchronization. However, DORE can further cut down the remaining $50\%$ communication. 

Taking the blockwise $p$-norm quantization as an example, every element of $\vx$ can be represented by $\frac{3}{2}$ bits using the simple ternary coding $\{0,\pm1\}$, along with one magnitude for each block. 
For example, if we consider the uniform block size $b$, the number of bits to represent a $d$-dimension vector of $32$ bit float-point numbers can be reduced from $32d$ bits to $32\frac{d}{b}+\frac{3}{2}d$ bits. As long as the block size $b$ is relatively large with respect to the constant $32$, the cost $32\frac{d}{b}$ for storing the float-point number is relatively small such that the compression rate is close to $32d/({\frac{3}{2}}d) \approx 21.3$ times (for example, $19.7$ times when $b=256$). 

Applying this quantization, QSGD, Terngrad, MEM-SGD, and DIANA need to transmit $(32d+32\frac{d}{b}+\frac{3}{2}d)$ bits per iteration and thus they are able to cut down $47\%$ of the overall $2\times 32d$ bits per iteration through gradient compression when $b=256$. But with DORE, we only need to transmit $2(32\frac{d}{b}+\frac{3}{2}d)$ bits per iteration. Thus DORE can reduce over $95\%$ of the total communication by compressing both the gradient and model transmission. More efficient coding techniques such as Elias coding~\citep{eliascoding} can be applied to further reduce the number of bits per iteration.

\section{Convergence Analysis}
\label{sec:theory}
To show the convergence of DORE, we make the following commonly used assumptions.

\begin{assumption}\label{asm1} Each worker node samples an unbiased estimator of the gradient stochastically with bounded variance, i.e.,
for $i=1,2,\cdots, n$ and $\forall \vx\in\RR^{d}$,
\begin{equation}\label{gradient}
    \EE[\vg_{i}|\vx]=\nabla f_{i}(\vx),\quad \EE\|\vg_{i}-\nabla f_{i}(\vx)\|^{2}\leq \sigma_{i}^{2},
\end{equation}
where $\vg_{i}$ is the estimator of $\nabla f_i$ at $\vx$. In addition, we define $\sigma^{2}=\frac{1}{n}\sum_{i=1}^{n}\sigma_{i}^{2}$.
\end{assumption}

\begin{assumption}\label{asm2}Each $f_{i}$ is $L$-Lipschitz differentiable, i.e., for $i=1,2,\cdots, n$ and $\forall \vx,\vy\in\RR^{d}$,
\begin{equation}\label{lipschitz}
   \textstyle f_{i}(\vx)\leq f_{i}(\vy)+\dotp{\nabla f_{i}(\vy), \vx-\vy}+\frac{L}{2}\|\vx-\vy\|^{2}.
\end{equation}
\end{assumption}

\begin{assumption}\label{asm3} Each $f_{i}$ is $\mu$-strongly convex ($\mu\geq 0$), i.e., for $i=1,2,\cdots, n$ and $\forall \vx,\vy\in\RR^{d}$,
\begin{equation}\label{strcov}
   \textstyle f_{i}(\vx)\geq f_{i}(\vy)+\dotp{\nabla f_{i}(\vy), \vx-\vy}+\frac{\mu}{2}\|\vx-\vy\|^{2}.
\end{equation}
\end{assumption}

For simplicity, we use the same compression operator for all worker nodes, and the master node can apply a different compression operator.
We denote the constants in Assumption~\ref{ass:compression} as $C_q$ and $C_q^m$ for the worker and master nodes, respectively.
Then we set $\alpha$ and $\beta$ in both algorithms to satisfy
\begin{align}\label{alpha-beta}
   \textstyle { \frac{1- \sqrt{1-{4C_q(C_q+1)\over nc}}}{2(C_q+1)} \leq} &\ \textstyle{\alpha \leq  {1+ \sqrt{1-{4C_q(C_q+1)\over nc}}\over 2(C_q+1)}},\nonumber\\
   \textstyle{0<}&\ \textstyle{\beta\leq \frac{1}{C_{q}^m+1}}, 
   \end{align}
with $c\geq\frac{4C_{\vq}(C_{q}+1)}{n}$. 
We consider two scenarios in the following two subsections: $f$ is strongly convex with a convex regularizer $R$ and $f$ is non-convex with $R=0.$

\subsection{The strongly convex case}

\begin{theorem}\label{thm1}
Under Assumptions~\ref{ass:compression}-\ref{asm3}, if $\alpha$ and $\beta$ in Algorithm~\ref{algocomplete} satisfy~\eqref{alpha-beta}, $\eta$ and $\gamma$ satisfy
\begin{align}
& \eta <\min \textstyle\left({-C_q^m+ \sqrt{(C_q^m)^2+4(1-(C_q^m+1)\beta)}\over 2C_q^m},\right. \textstyle\left.\frac{4\mu L}{(\mu+L)^2(1+c\alpha)-4\mu L}\right), \\
& \textstyle{\eta(\mu+L)\over 2(1+\eta)\mu L} \leq \textstyle\gamma\leq {2\over (1+c\alpha)(\mu+L)},
\end{align}
then we have
\begin{align}\label{linear-conv}
\textstyle{\vV^{k+1}\leq \rho^k\vV^1 + \frac{(1+\eta)(1+nc\alpha)}{n(1-\rho)}\beta\gamma^2\sigma^2,}
\end{align}
with 
\begin{align*}
    \textstyle{\vV^k=}&\textstyle{\beta(1-(C_q^m+1)\beta)\EE\|\vq^{k-1}\|^2+\EE\|\hat{\vx}^{k}-\vx^{*}\|^{2}} + 
    \textstyle{+\frac{(1+\eta)c\beta\gamma^{2}}{n}\sum_{i=1}^{n}\EE\|\vh_{i}^{k}-\nabla f_{i}(\vx^{*})\|^{2}},
\end{align*}
and 
\begin{align*}
\rho =  &\textstyle\max \Big( \frac{(\eta^2+\eta)C_q^m}{1-(C_q^m+1)\beta}, 1+\eta\beta-\frac{2(1+\eta)\beta\gamma\mu L}{\mu+L}, 1-\alpha \Big) < 1.
\end{align*}

\end{theorem}

\begin{corollary}\label{cor1}
When there is no error compensation and we set $\eta=0$, then $\rho=\max (1-\frac{2\beta\gamma\mu L}{\mu+L}, 1-\alpha)$.
If we further set
\begin{equation}\label{alpha-beta2}
   \textstyle \alpha=\frac{1}{2(C_{q}+1)},\quad \beta=\frac{1}{C_{q}^m+1},\quad c=\frac{4C_{q}(C_{q}+1)}{n},
\end{equation}
and choose the largest step-size 
$\gamma=\frac{2}{(\mu+L)(1+2C_{q}/n)},$
the convergent factor is
\begin{equation}\label{factor}
   \textstyle (1-\rho)^{-1}=\max\Big(2(C_{q}+1), (C_{q}^m+1)\frac{(\mu+L)^{2}}{2\mu L}\Big(\frac{1}{2}+\frac{C_{q}}{n}\Big)\Big).
\end{equation}
\end{corollary}

\begin{remark}
In particular, suppose $\{\Delta_i\}_{i=1}^{n}$ are compressed using the Bernoulli $p$-norm quantization with the largest block size $d_{\max}$, then
%
$C_{q}={1\over \alpha^w} -1,$
with $\alpha^w =\min_{\vzero\neq \vx \in\RR^{d_{\max}}}{\|\vx\|_2^2\over \|\vx\|_1\|\vx\|_p}\leq 1$.
Similarly, $\vq$ is compressed using the Bernoulli $p$-norm quantization with $C_q^m={1\over \alpha^m}-1$.
Then the linear convergent factor is
\begin{equation}\label{confactor_DRC}
    \textstyle (1-\rho)^{-1} = \max\Big\{\frac{2}{\alpha^w}, \frac{1}{\alpha^m}\frac{(\mu+L)^{2}}{\mu L}\Big(\frac{1}{2}-\frac{2}{n}+\frac{2}{n\alpha^w}\Big)\Big\}.
\end{equation}
While the result of DIANA in~\citep{mishchenko2019distributed} is 
$\max\Big\{\frac{2}{\alpha^w}, \frac{\mu+L}{\mu }\Big(\frac{1}{2}-\frac{1}{n}+\frac{1}{n\alpha^w}\Big)\Big\},$
which is larger than~\eqref{confactor_DRC} with $\alpha^m=1$ (no compression for the model). 
When there is no compression for $\Delta_i$, i.e., $\alpha^w=1$, the algorithm reduces to the gradient descent, and the linear convergent factor is the same as that of the gradient descent for strongly convex functions.
\end{remark}

\begin{remark}
Although error compensation often improves the convergence in practice, in theory, no compensation, i.e., $\eta=0$, provides the best convergence rate. This is because we don't have much information of the error being compensated. Filling this gap will be an interesting future direction.
\end{remark}

\subsection{The nonconvex case with $R=0$}\label{sec42}

\begin{theorem}\label{thm3}
Under Assumptions~\ref{ass:compression}-\ref{asm2} and the additional assumption that each worker samples the gradient from the full dataset, we set $\alpha$ and $\beta$ according to~\eqref{alpha-beta}.
By choosing
$$\textstyle{\gamma \leq \min\Big\{\frac{-1+\sqrt{1+ \frac{48L^2\beta^2(C_q^m+1)^2}{C_q^m}}}{12L\beta(C_q^m+1)}, \frac{1}{6L\beta(1+c\alpha)(C_q^m+1)}\Big\},}$$
we have
\begin{align}\label{nonconvex-conv}
  &\frac{\frac{\beta}{2}-3(1+c\alpha)(C_{q}^{m}+1)L\beta^2\gamma}{K}\sum_{k=1}^{K} \EE\|\nabla f(\hat{\vx}^{k})\|^{2} \nonumber\\ 
  \leq& \frac{\Lambda^1-\Lambda^{K+1}}{\gamma K}
  + \frac{3(C^m_{q}+1)(1+nc\alpha)L\beta^2\sigma^2\gamma}{n},
\end{align}
where 
\begin{align}
\Lambda^{k}=&(C_q^m+1)L\beta^2\|\vq^{k-1}\|^2+f(\hat{\vx}^{k})-f^{*}+3c(C_{q}^{m}+1)L\beta^{2}\gamma^2\frac{1}{n}\sum_{i=1}^{n}\EE\|\vh^{k}_{i}\|^{2}.
\end{align}
\end{theorem}

\begin{corollary}\label{col2}
Let $\alpha={1\over 2(C_q+1)}, \beta=\frac{1}{C_q^m+1}$, and $c=\frac{4C_q(C_q+1)}{n}$, then $1+nc\alpha$ is a fixed constant. If $\gamma = \frac{1}{12L(1+c\alpha)(1+\sqrt{K/n})}$, when K is relatively large, we have 
\begin{align}\label{col_nonconv_eq}
\frac{1}{K}\sum_{k=1}^K \EE\|\nabla f(\hat{\vx}^{k})\|^{2}\lesssim {1\over K} + {1\over \sqrt{Kn}}.
\end{align}
\end{corollary}

\begin{remark}
The dominant term in~\eqref{col_nonconv_eq} is $O(1/\sqrt{Kn})$, which implies that the sample complexity of each worker node is $O(1/(n\epsilon^{2}))$ in average to achieve an $\epsilon$-accurate solution. 
It shows that, same as DoubleSqueeze in~\cite{tang2019doublesqueeze}, DORE is able to perform linear speedup. 
Furthermore, this convergence result is the same as the P-SGD without compression. 
Note that DoubleSqueeze has an extra term $(1/K)^{\frac{2}{3}}$, and its convergence requires the bounded variance of the compression operator.
\end{remark}

\begin{table*}[!ht]
\begin{center}
\begin{tabular}{|c|c|c|c|c|}\hline
    Algorithm  & Compression & Compression Assumed & Linear rate  & Nonconvex Rate   \\\hline 
    QSGD & Grad  & $2$-norm Quantization & N/A & $\frac{1}{K}+B $ \\\hline
    DIANA & Grad & $p$-norm Quantization & \checkmark & ${1\over\sqrt{Kn}}+{1\over K}$\\\hline
    DoubleSqueeze & Grad+Model & Bounded Variance & N/A  & ${1\over\sqrt{Kn}}+{1\over K^{2/3}}+{1\over K}$ \\\hline
    DORE & Grad+Model & Assumption~\ref{ass:compression} & \checkmark & ${1\over\sqrt{Kn}}+{1\over K}$ \\\hline
\end{tabular}
\end{center}
\caption{A comparison between related algorithms. 
DORE is able to converges linearly to the $\mathcal{O} (\sigma)$ neighborhood of optimal point like full-precision SGD and DIANA in the strongly convex case while achieving much better communication efficiency. DORE also admits linear speedup in the nonconvex case like DoubleSqueeze but DORE doesn't require the assumptions of bounded compression error or bounded gradient.
}
\label{table}
\end{table*}

\section{Experiment}
\label{sec:experiment}
In this section, we validate the theoretical results and demonstrate the superior performance of DORE. 
Our experimental results demonstrate that (1) DORE achieves similar convergence speed as full-precision SGD and state-of-art quantized SGD baselines and (2) its iteration time is much smaller than most existing algorithms, supporting the superior communication efficiency of DORE. 

To make a fair comparison, we choose the same Bernoulli $\infty$-norm quantization as described in Section~\ref{sec:alg} and the quantization block size is 256 for all experiments if not being explicitly stated because $\infty$-norm quantization is unbiased and commonly used. The parameters $\alpha, \beta, \eta$ for DORE are chosen to be $0.1, 1 ~\text{and}~ 1$, respectively. 

The baselines we choose to compare include SGD, QSGD~\citep{alistarh2017qsgd}, MEM-SGD~\citep{Stich:2018:SSM:3327345.3327357}, DIANA~\citep{mishchenko2019distributed}, DoubleSqueeze and DoubleSqueeze (topk) ~\citep{tang2019doublesqueeze}. 
SGD is the vanilla SGD without any compression and QSGD quantizes the gradient directly. MEM-SGD is the  QSGD with error compensation. DIANA, which only compresses and transmits the gradient difference, is a special case of the proposed DORE. DoubleSqueeze quantizes both the gradient on the workers and the averaged gradient on the server with error compensation. Although DoubleSqueeze is claimed to work well with both biased and unbiased compression, in our experiment it converges much slower and suffers the loss of accuracy with unbiased compression. Thus, we also compare with DoubleSqueeze using the Top-k compression as presented in~\cite{tang2019doublesqueeze}.

\subsection{Strongly convex}
To verify the convergence for strongly convex and smooth objective functions, we conduct the experiment on a linear regression problem: $f(\vx)=\|\vA\vx-\vb\|^2+\lambda \|\vx\|^2$. The data matrix $\vA \in \mathbb{R}^{1200\times 500}$ and optimal solution $\vx_* \in \mathbb{R}^{500}$ are randomly synthesized. Then we generate the prediction $\vb$ by sampling from a Gaussian distribution whose mean is $\vA\vx_*$. The rows of the data matrix $\vA$ are allocated evenly to 20 worker nodes. To better verify the linear convergence to the $\mathcal{O}(\sigma)$ neighborhood around the optimal solution, we take the full gradient in each node for all algorithms to exclude the effect of the gradient variance ($\sigma=0$).

As showed in Figure~\ref{fig:lr}, 
with full gradient and a constant learning rate, DORE  converges linearly, same as SGD and DIANA, but QSGD, MEM-SGD, DoubleSqueeze, as well as DoubleSqueeze (topk) converge to a neighborhood of the optimal point. This is because these algorithms assume the bounded gradient and their convergence errors depend on that bound. 
Although they converge to the optimal solution using a diminishing step size, their converge rates will be much slower. 

In addition, we also validate that the norms of the gradient and model residual decrease exponentially, and it explains the linear convergence behavior of DORE. For more details, please refer to Appendix~\ref{sec:residual}.

\begin{figure}[!bt]
    \centering
    \vspace{-0.2in}
    \includegraphics[width=0.5\textwidth]{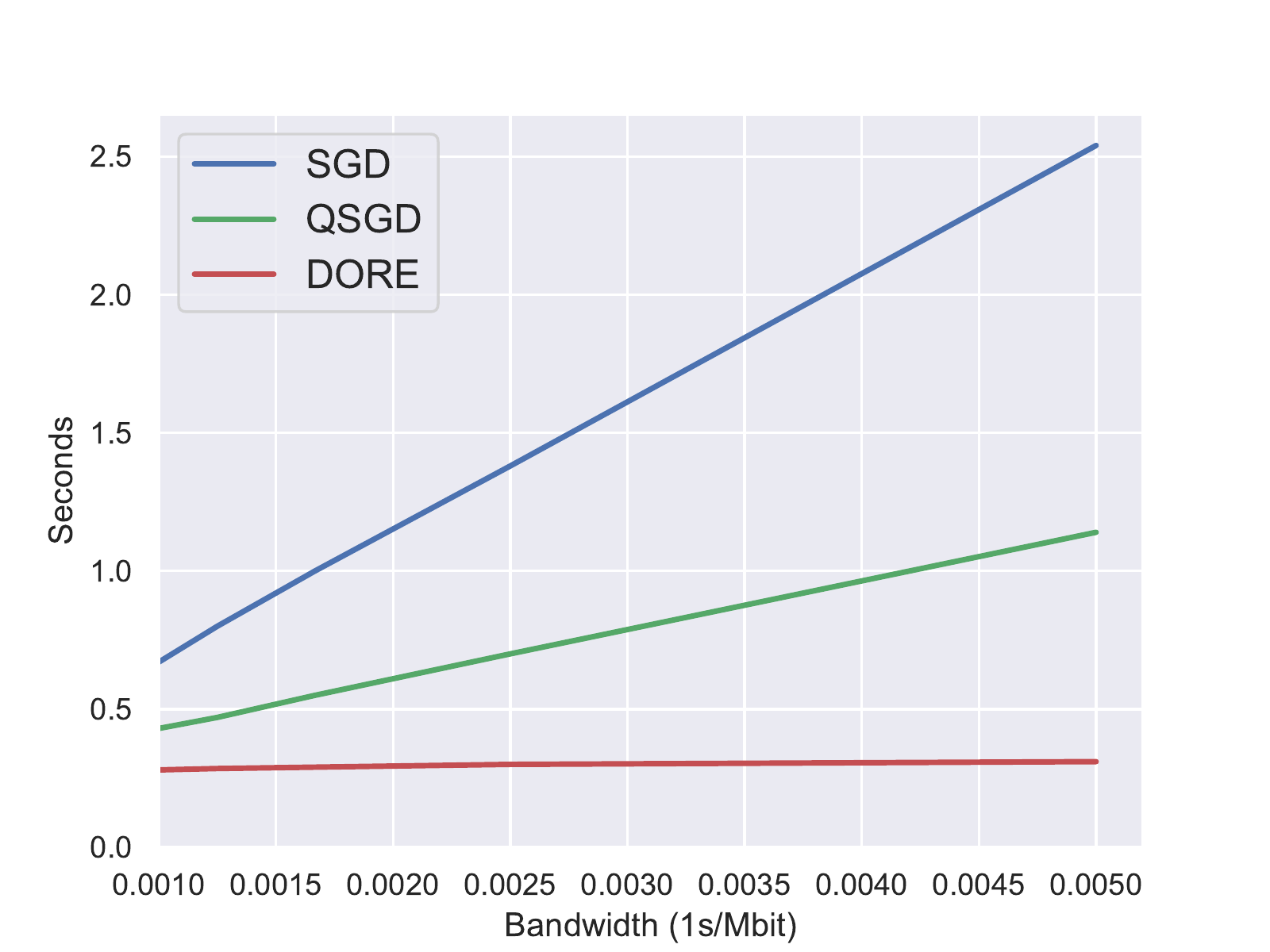}
    \caption{Per iteration time cost on Resnet18 for SGD, QSGD, and DORE. It is tested in a shared cluster environment connected by Gigabit Ethernet interface. DORE speeds up the training process significantly by mitigating the communication bottleneck.}
    \label{fig:time}
    \vspace{-0.1in}
\end{figure}

\subsection{Nonconvex}
To verify the convergence in the nonconvex case, we test the proposed DORE with two classical deep neural networks on two representative datasets, respectively, i.e., LeNet~\citep{lenet1998} on MNIST and Resnet18~\citep{He2016DeepRL} on CIFAR10. 
In the experiment, we use 1 parameter server and 10 workers, each of which is equipped with an NVIDIA Tesla K80 GPU.
The batch size for each worker node is 256.
We use 0.1 and 0.01 as the initial learning rates for LeNet and Resnet18, and decrease them by a factor of 0.1 after every 25 and 100 epochs, respectively. 
All parameter settings are the same for all algorithms.

\begin{figure*}[!ht]
\vspace{-0.3in}
\subfloat[Learning rate=0.05]{
\begin{minipage}[t]{0.45\textwidth}
    \centering
    \includegraphics[width=0.9\textwidth]{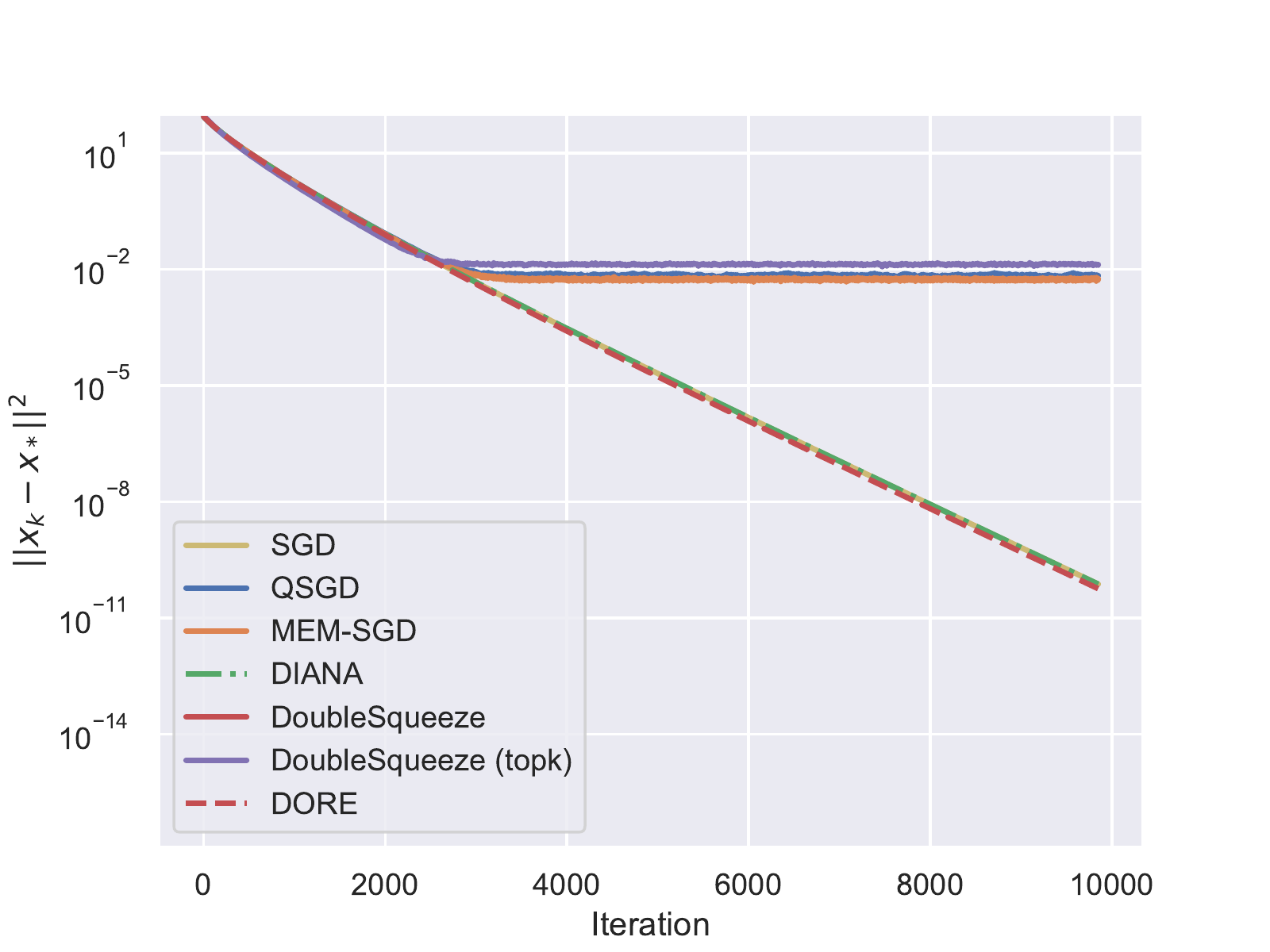}
\end{minipage}
}
\subfloat[Learning rate=0.025]{
\begin{minipage}[t]{0.45\textwidth}
    \centering
    \includegraphics[width=0.9\textwidth]{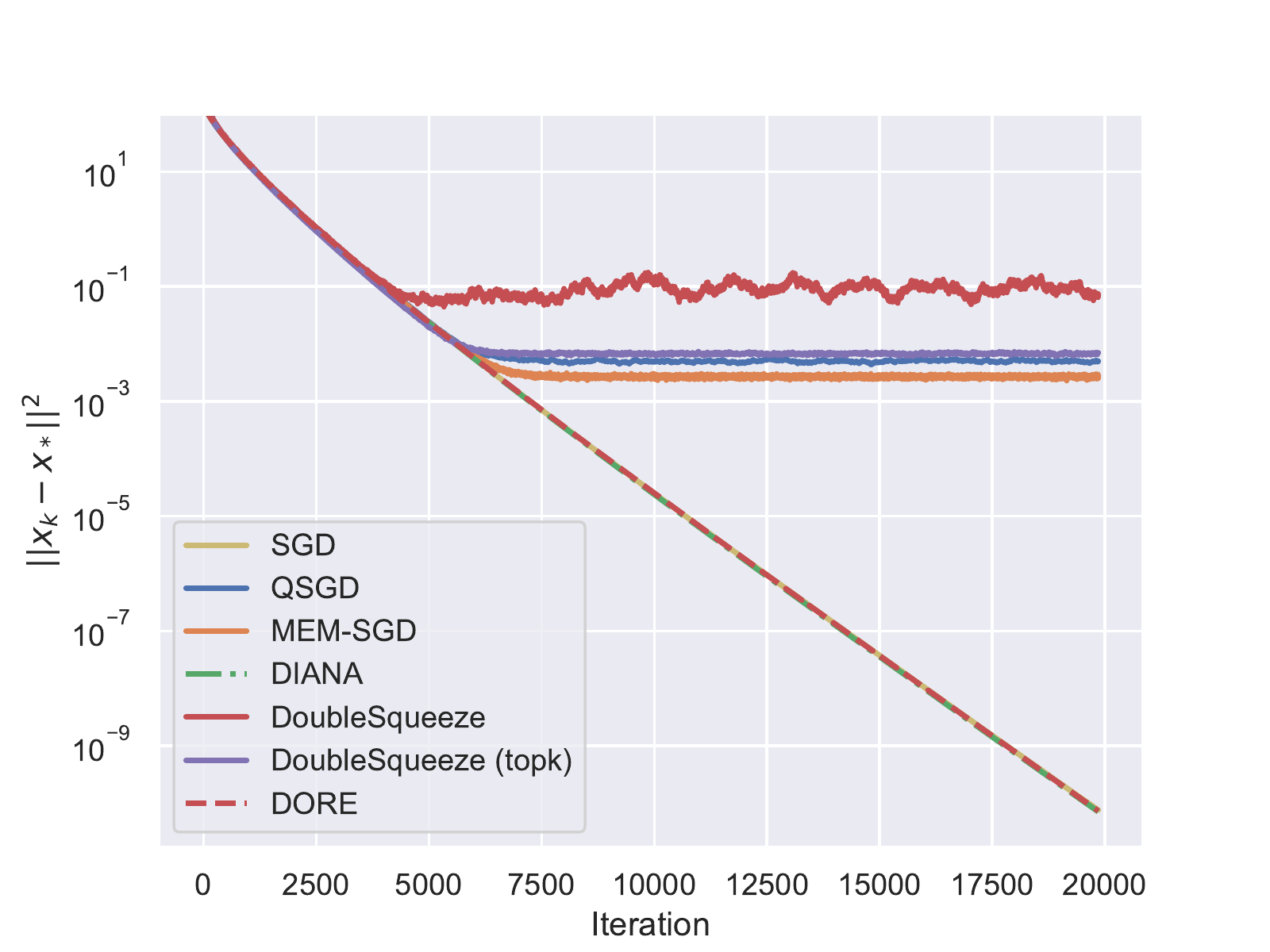}
\end{minipage}
}
\caption{Linear regression on synthetic data. When the learning rate is 0.05, DoubleSqueeze diverges. In both cases, DORE, SGD, and DIANA converge linearly to the optimal point, while QSGD, MEM-SGD, DoubleSqueeze, and DoubleSqueeze (topk) only converge to the neighborhood even when full gradient is available.}
\vspace{-0.3in}
\label{fig:lr}
\end{figure*}

\begin{figure*}[!ht]
\subfloat[Training loss]{
\begin{minipage}[t]{0.45\textwidth}
    \centering
    \includegraphics[width=0.9\textwidth]{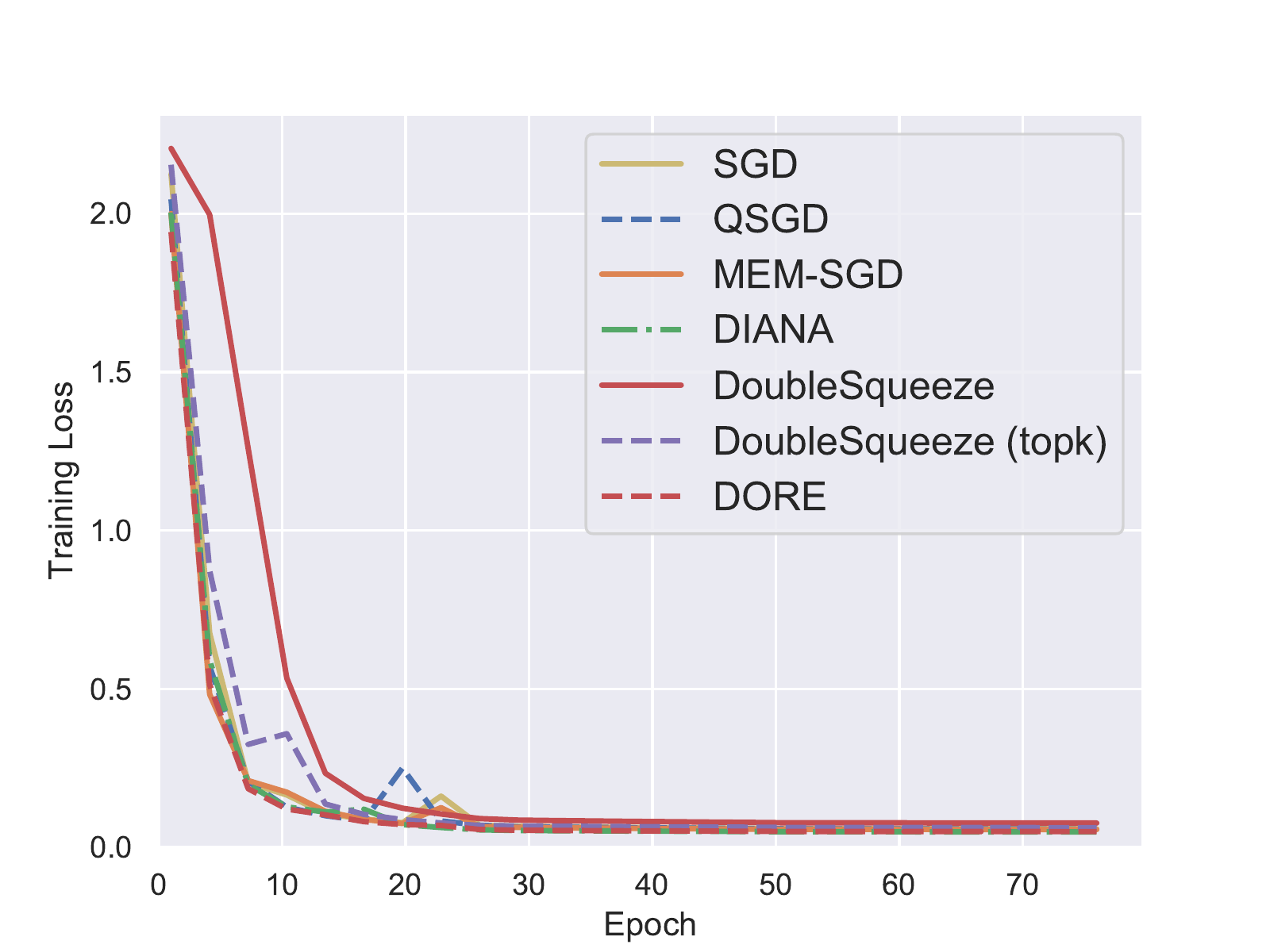}
\end{minipage}
}
\subfloat[Test Loss]{
\begin{minipage}[t]{0.45\textwidth}
    \centering
    \includegraphics[width=0.9\textwidth]{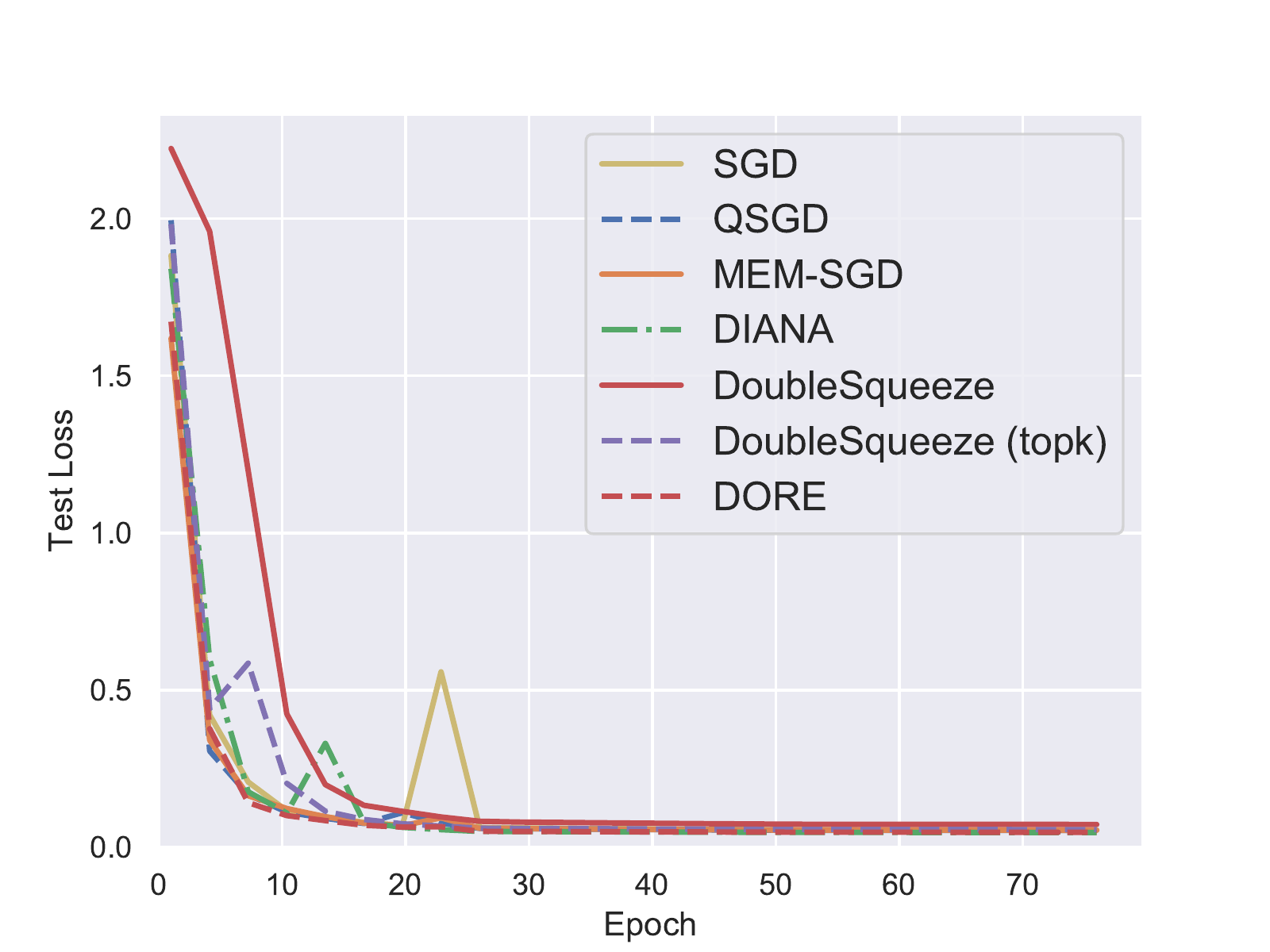}
\end{minipage}
}
\caption{LeNet trained on MNIST. DORE converges similarly as most baselines. It outperforms DoubleSqueeze using the same compression method while has similar performance as DoubleSqueeze (topk).}
\vspace{-0.3in}
\label{fig:mnist}
\end{figure*}

\begin{figure*}[!ht]
\subfloat[Training Loss]{
\begin{minipage}[t]{0.45\textwidth}
    \centering
    \includegraphics[width=0.9\textwidth]{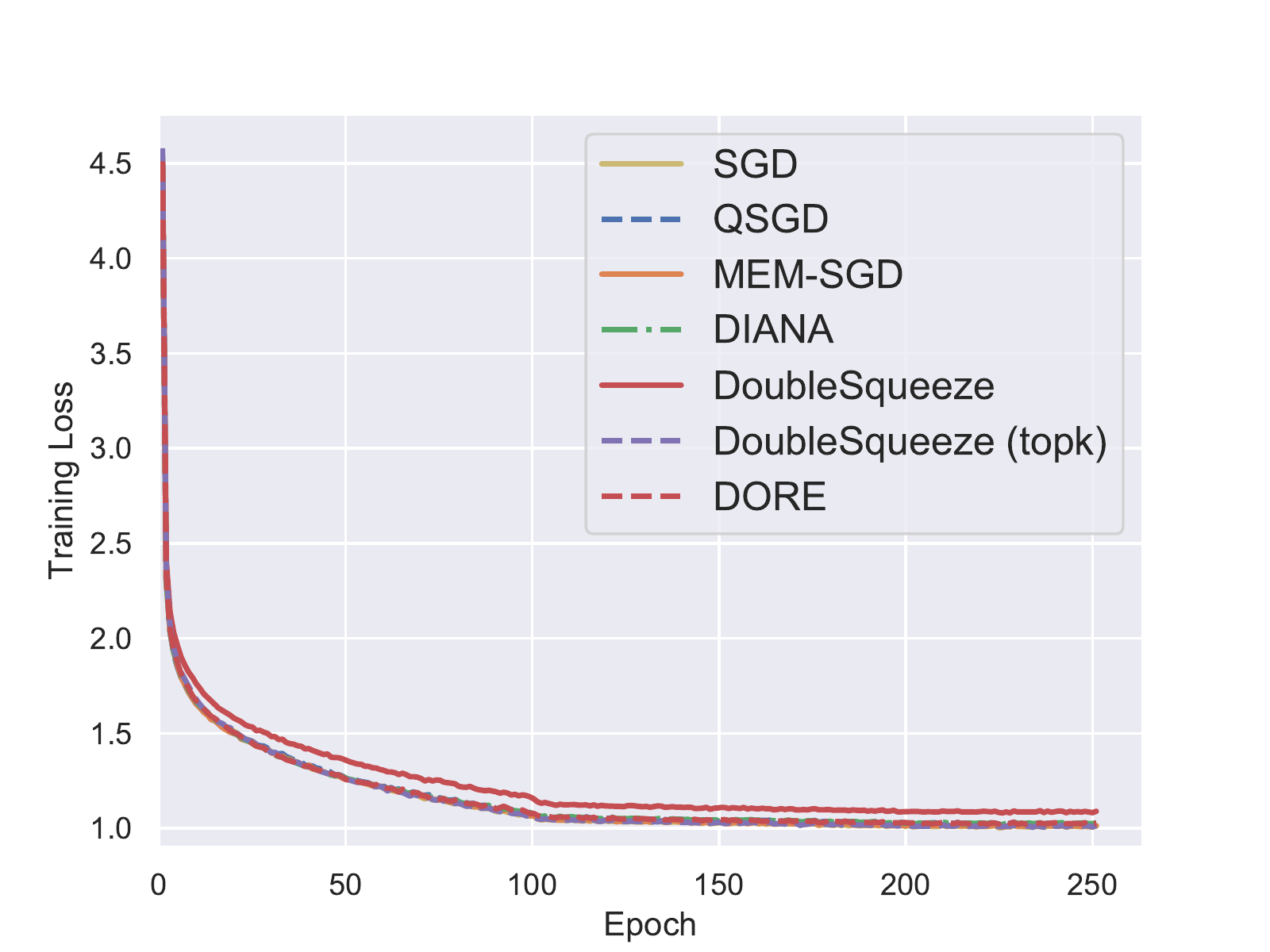}
\end{minipage}
}
\subfloat[Test Loss]{
\begin{minipage}[t]{0.45\textwidth}
    \centering
    \includegraphics[width=0.9\textwidth]{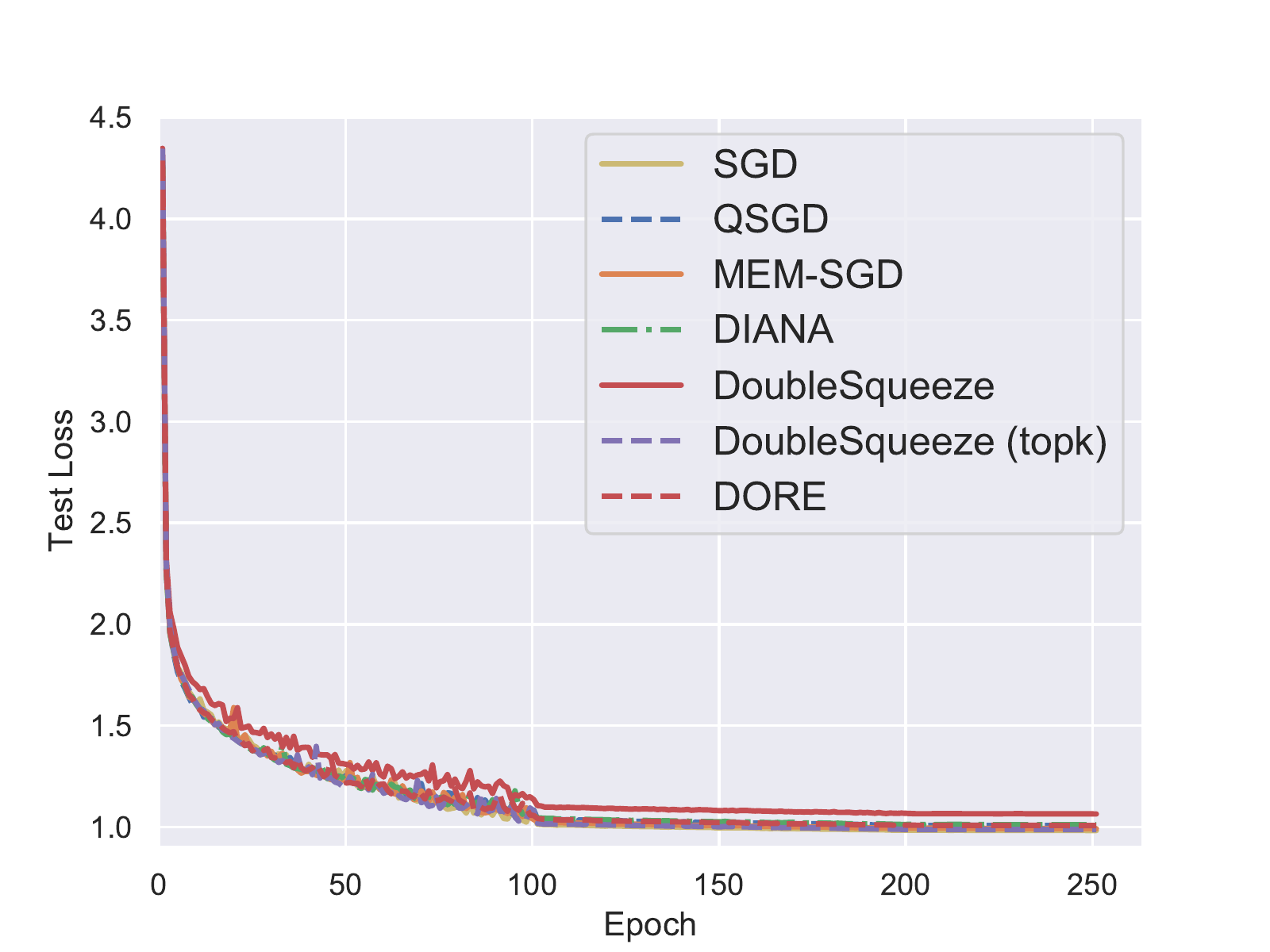}
\end{minipage}
}
\caption{Resnet18 trained on CIFAR10. DORE achieves similar convergence and accuracy as most baselines. DoubeSuqeeze converges slower and suffers from the higher loss but it works well with topk compression.
}
\vspace{-0.1in}
\label{fig:cifar}
\end{figure*}

Figures~\ref{fig:mnist} and~\ref{fig:cifar} show the training loss and test loss for each epoch during the training of LeNet on the MNIST dataset and Resnet18 on CIFAR10 dataset.  
The results indicate that in the nonconvex case, even with both compressed gradient and model information, DORE can still achieve similar convergence speed as full-precision SGD and other quantized SGD variants. DORE achieves much better convergence speed than DoubleSqueeze using the same compression method and converges similarly with DoubleSqueeze with Topk compression as presented in~\citep{tang2019doublesqueeze}. We also validate via parameter sensitivity in Appendix~\ref{sec:parameter} that DORE performs consistently well under different parameter settings such as compression block size, $\alpha, ~\beta$ and $\eta$.

\subsection{Communication efficiency}
In terms of communication cost, DORE enjoys the benefit of extremely efficient communication. As one example, under the same setting as the Resnet18 experiment described in the previous section, we test the time cost per iteration for SGD, QSGD, and DORE under varied network bandwidth. We didn't test MEM-SGD, DIANA, and DoubleSqueeze because  MEM-SGD, DIANA have similar time cost as QSGD while DoubleSqueeze has similar time cost as DORE. The result showed in Figure~\ref{fig:time} indicates that as the bandwidth becomes worse, with both gradient and model compression, the advantage of DORE becomes more remarkable compared to the baselines that don't apply compression for model synchronization .

\section{Conclusion}

Communication cost is the severe bottleneck for distributed training of modern large-scale machine learning models. 
Extensive works have compressed the gradient information to be transferred during the training process, but model compression is rather limited due to its intrinsic difficulty. In this paper, we proposed the Double Residual Compression SGD named DORE to compress both gradient and model communication that can  mitigate this bottleneck prominently. 
The theoretical analyses suggest good convergence rate of DORE under weak assumptions. 
Furthermore, DORE is able to reduce 95\% of the communication cost while maintaining similar convergence rate and model accuracy compared with the full-precision SGD.

\newpage
\bibliography{sections/reference}

\newpage
\appendix
\section{Supplementary materials}
\subsection{Compression error}
\label{sec:residual}
The property of the compression operator indicates that the compression error is linearly proportional to the norm of the variable being compressed:
\begin{equation*}
\EE \|Q(\vx)-\vx\|^{2}\leq C\|\vx\|^{2}.
\end{equation*}
We visualize the norm of the variables being compressed, i.e., the gradient residual (the worker side) and model residual (the master side) for DORE as well as error compensated gradient (the worker side) and averaged gradient (the master side) for DoubleSqueeze. As showed in Figure~\ref{fig:residual}, the gradient and model residual of DORE decrease exponentially and the compression errors vanish. However, for DoubleSqueeze, their norms only decrease to some certain value and the compression error doesn't vanish. It explains why algorithms without residual compression cannot converge linearly to the $\mathcal{O}(\sigma)$ neighborhood of the optimal solution in the strongly convex case.

\begin{figure*}[!ht]
\subfloat[Worker side]{
\begin{minipage}[t]{0.45\textwidth}
    \centering
    \includegraphics[width=0.9\textwidth]{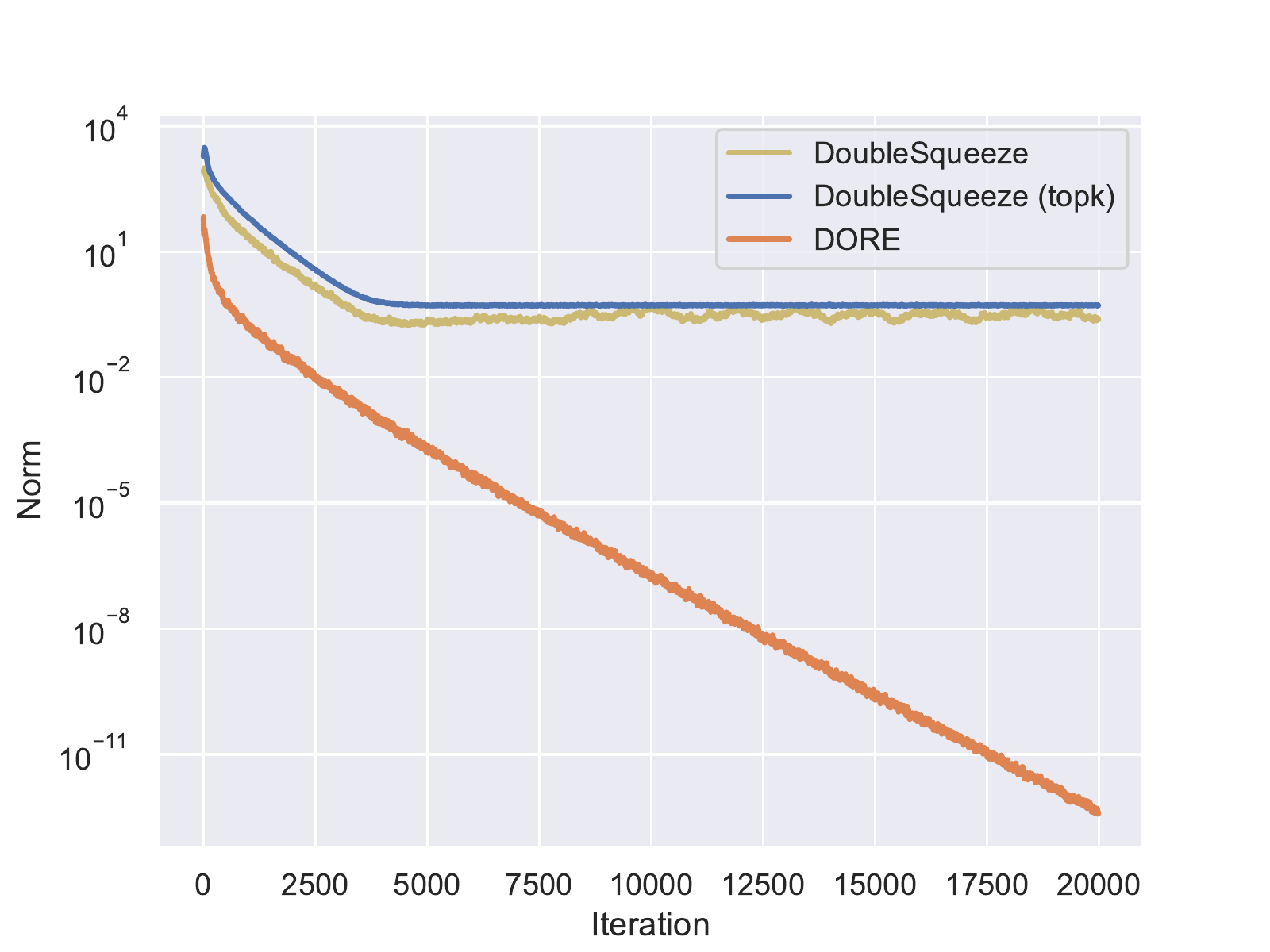}
\end{minipage}
}
\subfloat[Master side]{
\begin{minipage}[t]{0.45\textwidth}
    \centering
    \includegraphics[width=0.9\textwidth]{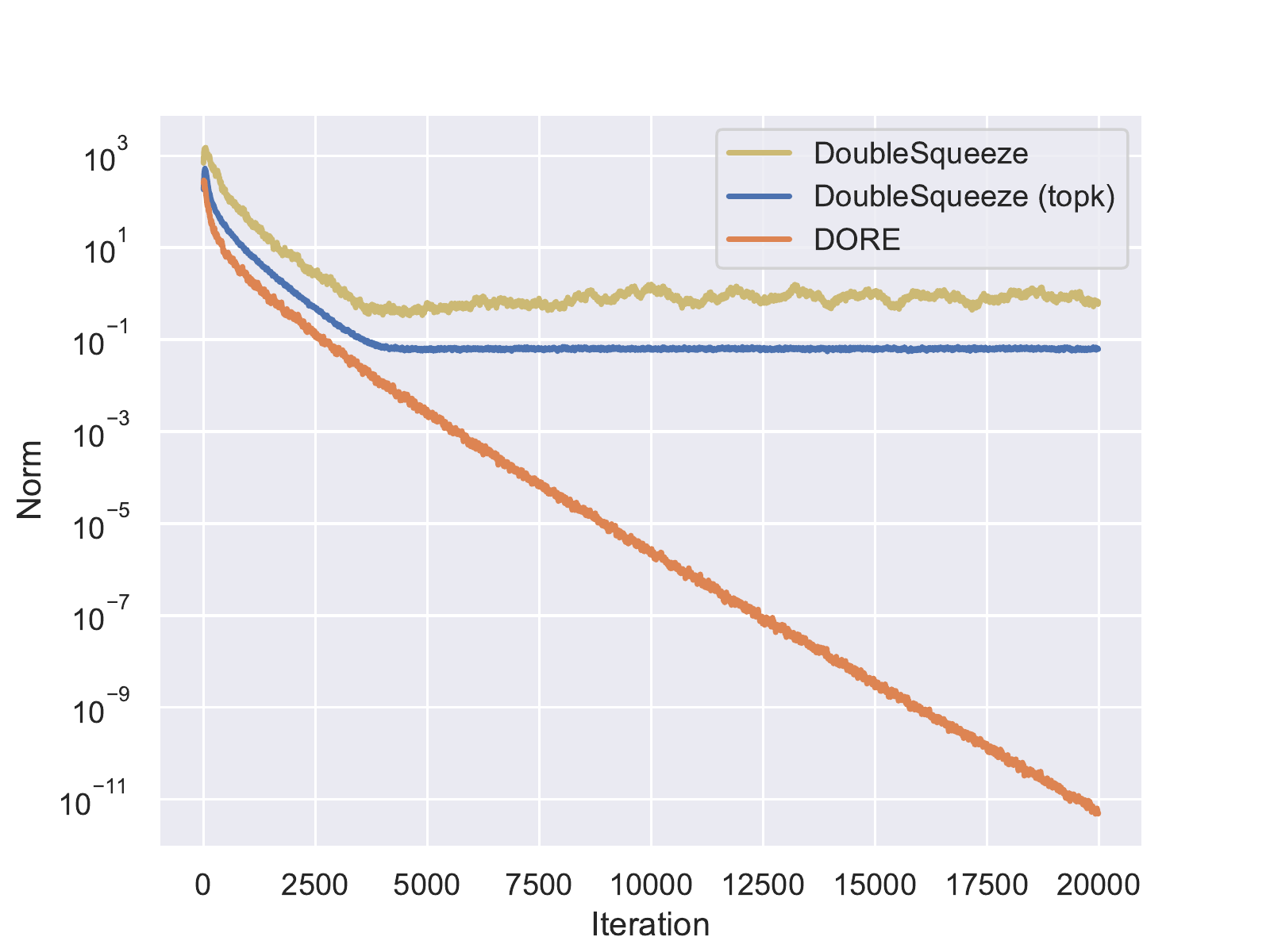}
\end{minipage}
}
\caption{The norm of variable being compressed in the linear regression experiment.} 
\label{fig:residual}
\end{figure*}

\subsection{Parameter sensitivity}\label{sec:parameter}

Continuing the MNIST experiment in Section~\ref{sec:experiment}, we further conduct parameter analysis on DORE. The basic setting for block size, learning rate, $\alpha$, $\beta$ and $\eta$ are 256, 0.1, 0.1, 1, 1, respectively. We change each parameter individually. Figures~\ref{fig:blocksize},~\ref{fig:alpha},~\ref{fig:beta}, and~\ref{fig:eta} demonstrate that DORE performs consistently well under different parameter settings.
\vspace{-0.2in}
\begin{figure*}[!ht]
\subfloat[Training loss]{
\begin{minipage}[t]{0.45\textwidth}
    \centering
    \includegraphics[width=0.9\textwidth]{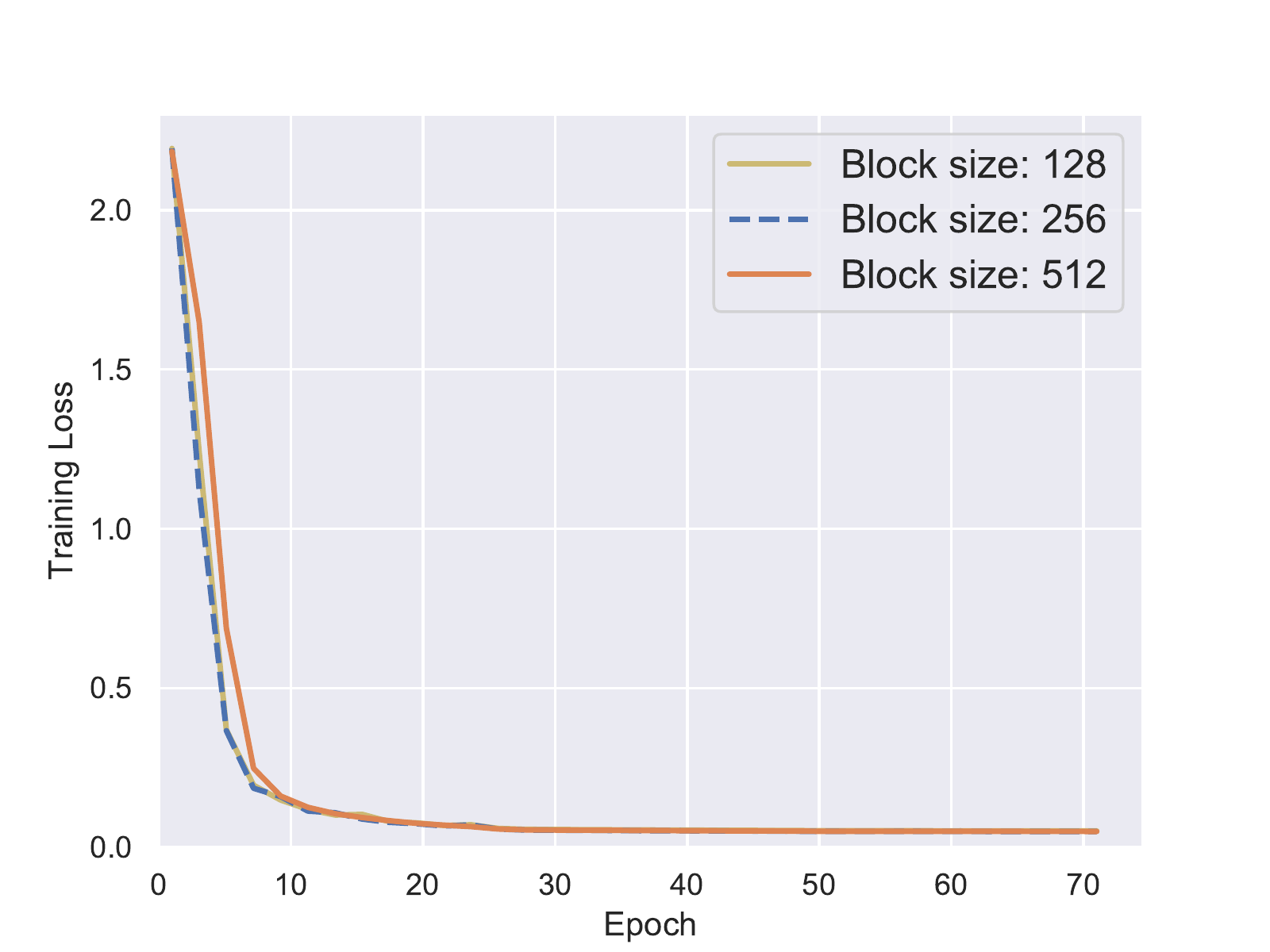}
\end{minipage}
}
\subfloat[Test loss]{
\begin{minipage}[t]{0.45\textwidth}
    \centering
    \includegraphics[width=0.9\textwidth]{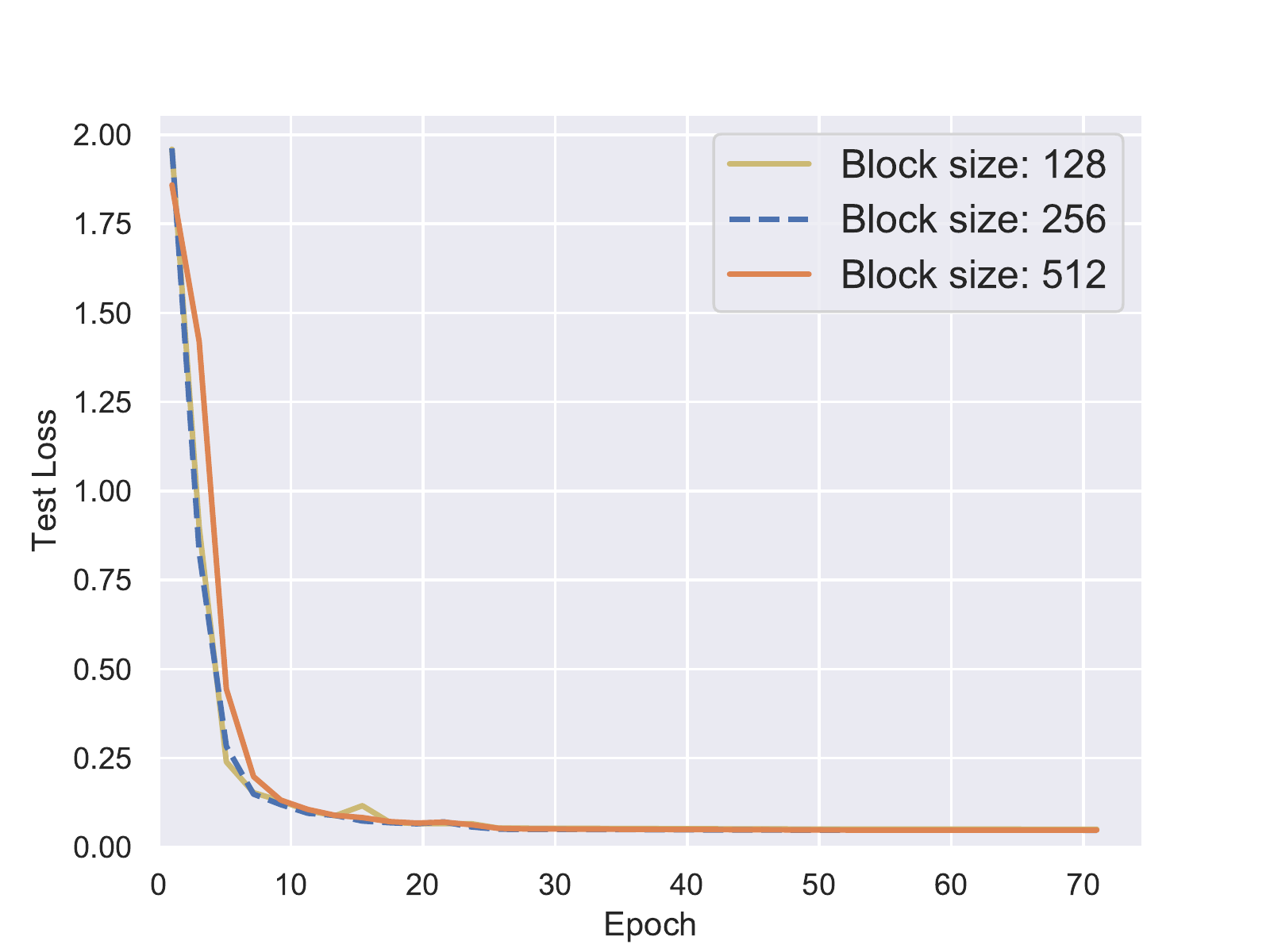}
\end{minipage}
}
\caption{Training under different compression block sizes.} 
\label{fig:blocksize}
\end{figure*}

\begin{figure*}[!ht]
\subfloat[Training loss]{
\begin{minipage}[t]{0.45\textwidth}
    \centering
    \includegraphics[width=0.9\textwidth]{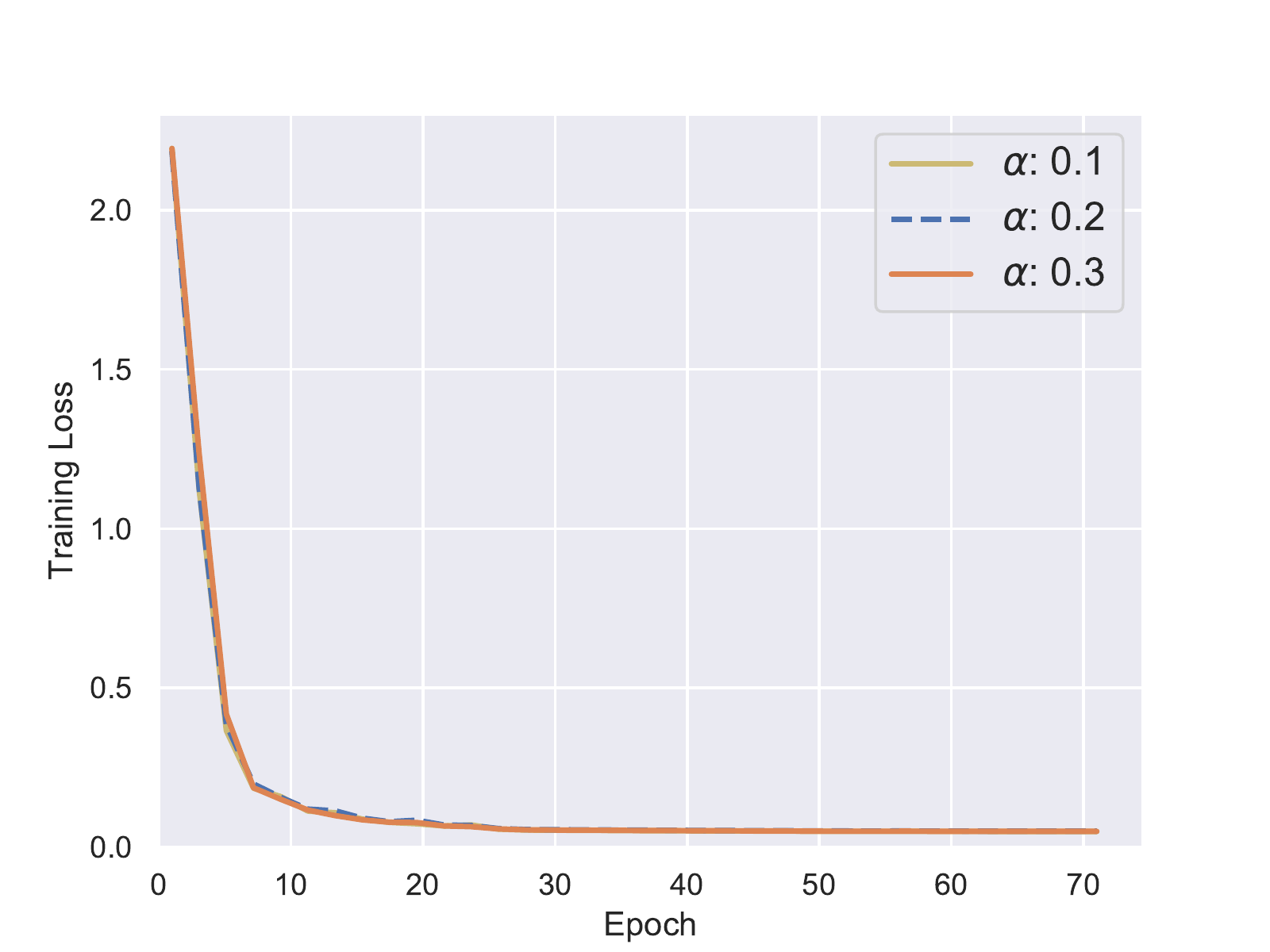}
\end{minipage}
}
\subfloat[Test loss]{
\begin{minipage}[t]{0.45\textwidth}
    \centering
    \includegraphics[width=0.9\textwidth]{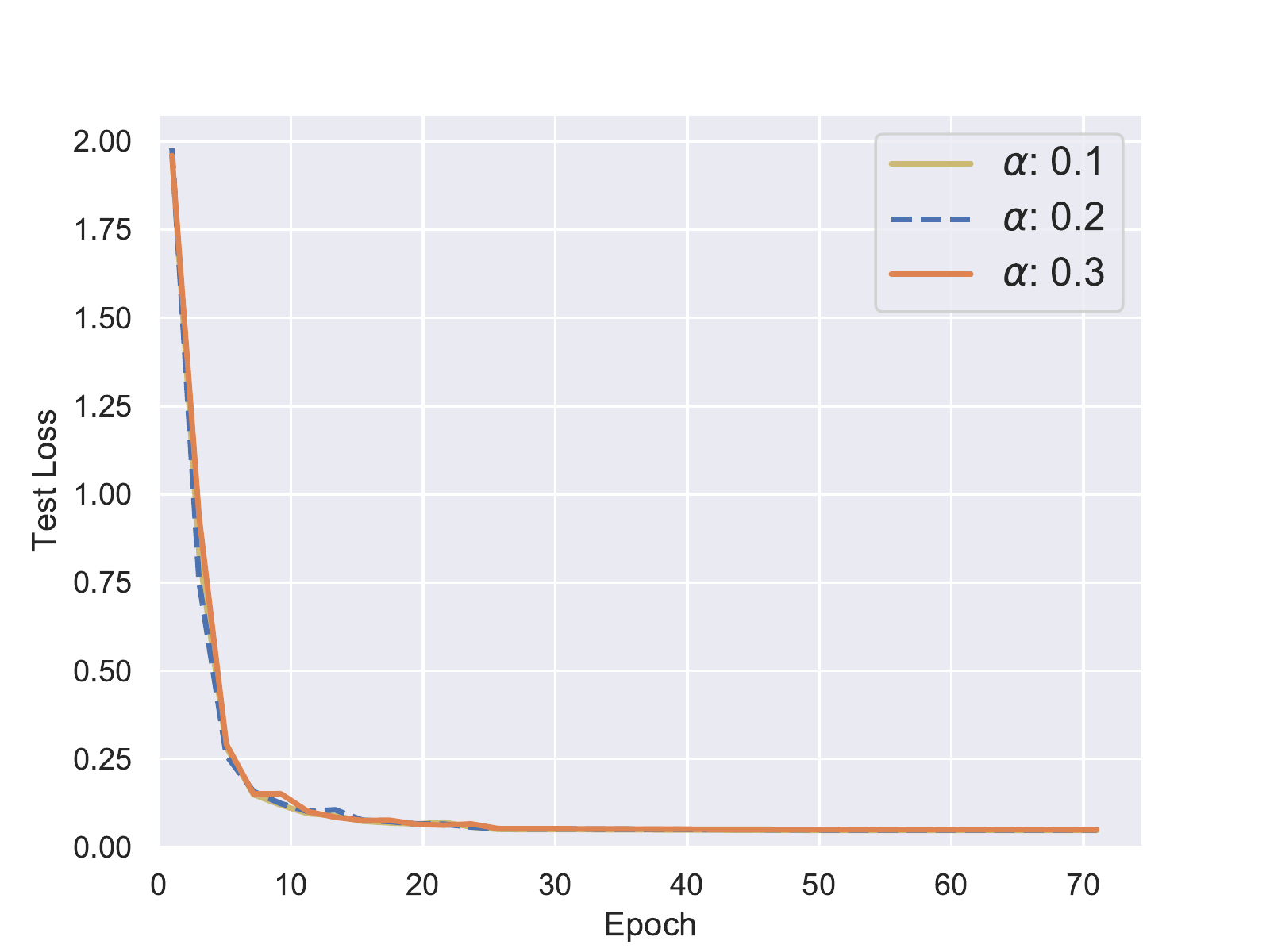}
\end{minipage}
}
\caption{Training under different $\alpha$} 
\label{fig:alpha}
\end{figure*}

\begin{figure*}[!ht]
\subfloat[Training loss]{
\begin{minipage}[t]{0.45\textwidth}
    \centering
    \includegraphics[width=0.9\textwidth]{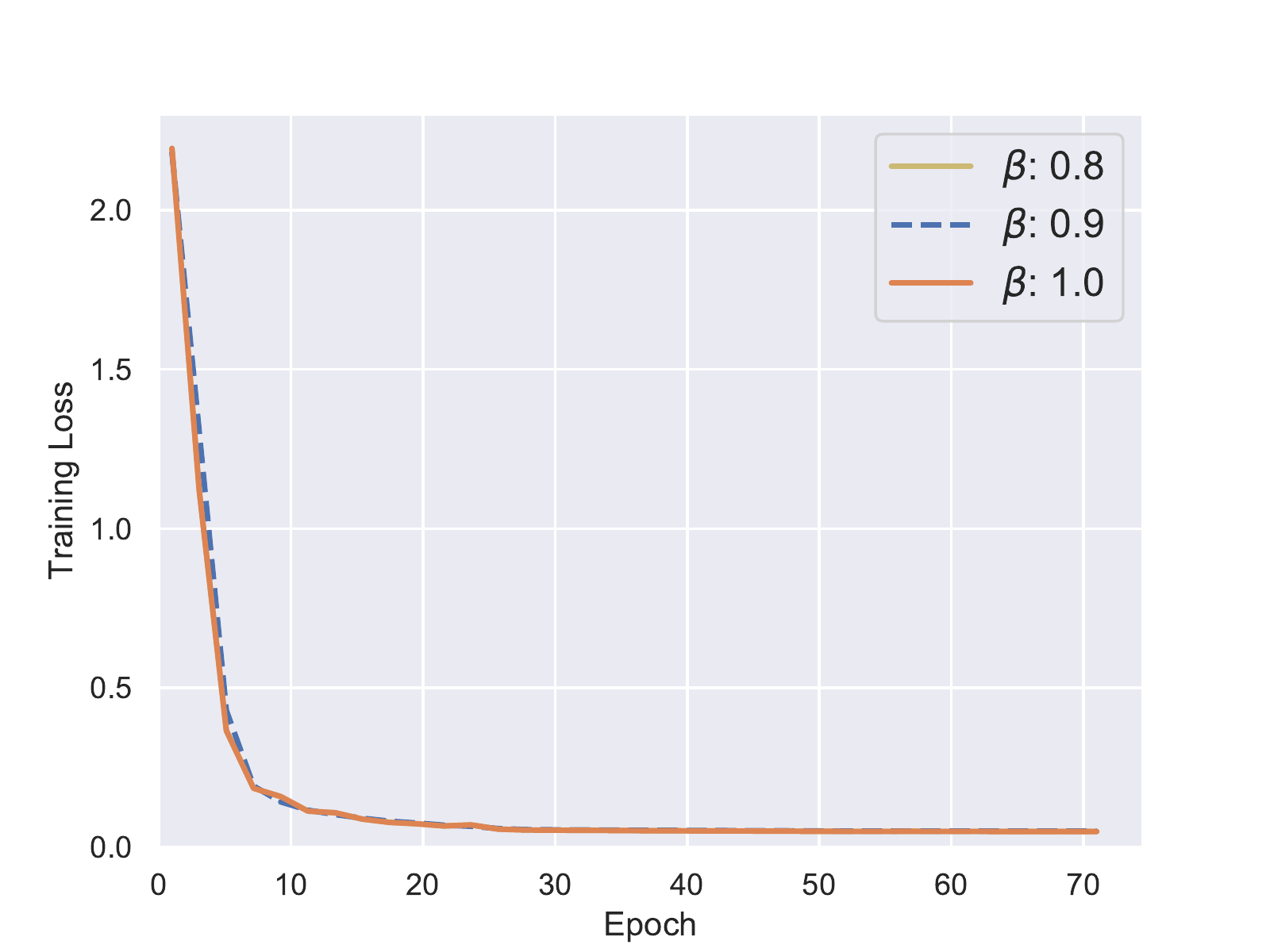}
\end{minipage}
}
\subfloat[Test loss]{
\begin{minipage}[t]{0.45\textwidth}
    \centering
    \includegraphics[width=0.9\textwidth]{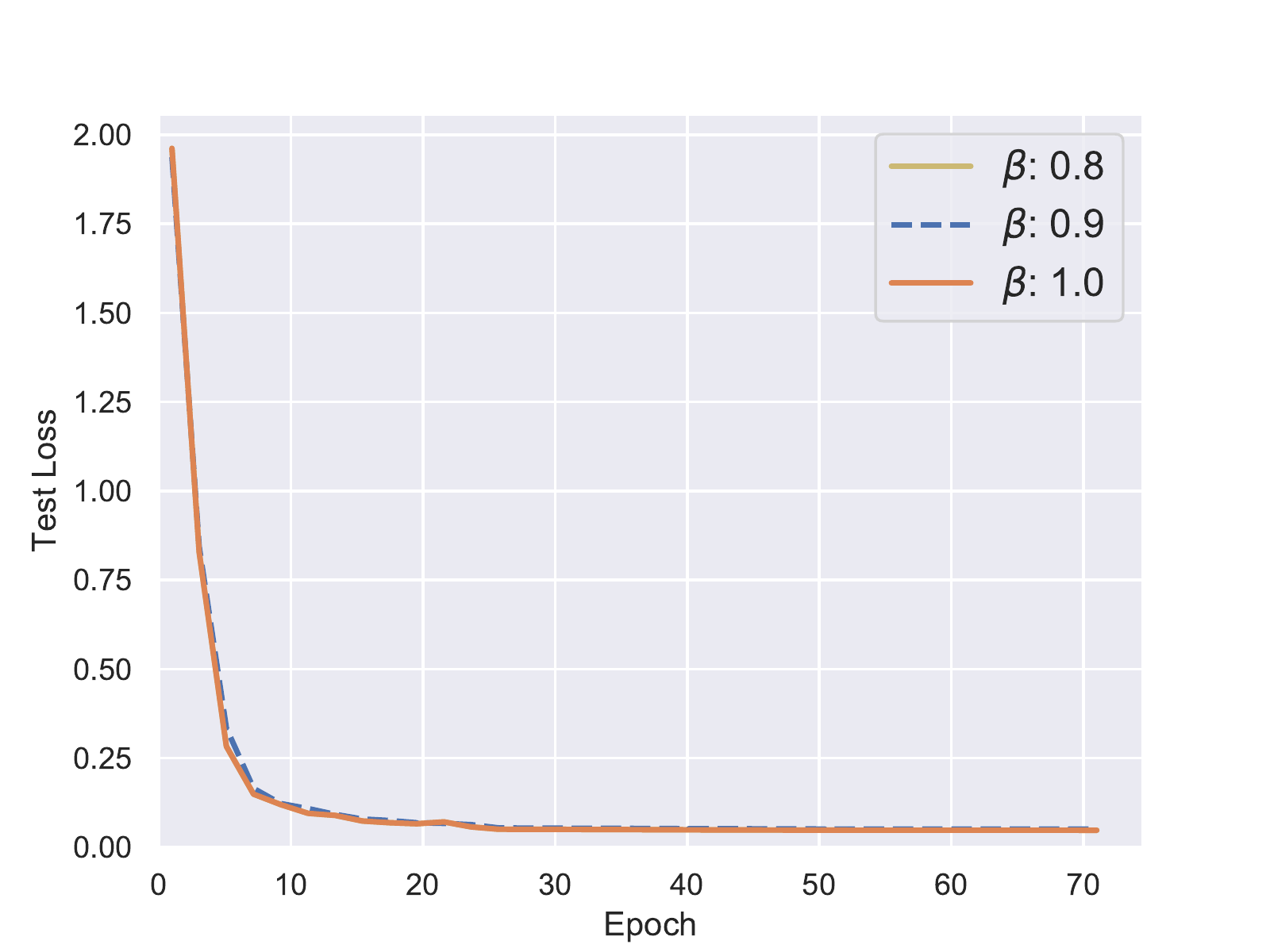}
\end{minipage}
}
\caption{Training under different $\beta$} 
\label{fig:beta}
\end{figure*}

\begin{figure*}[!ht]
\subfloat[Training loss]{
\begin{minipage}[t]{0.45\textwidth}
    \centering
    \includegraphics[width=0.9\textwidth]{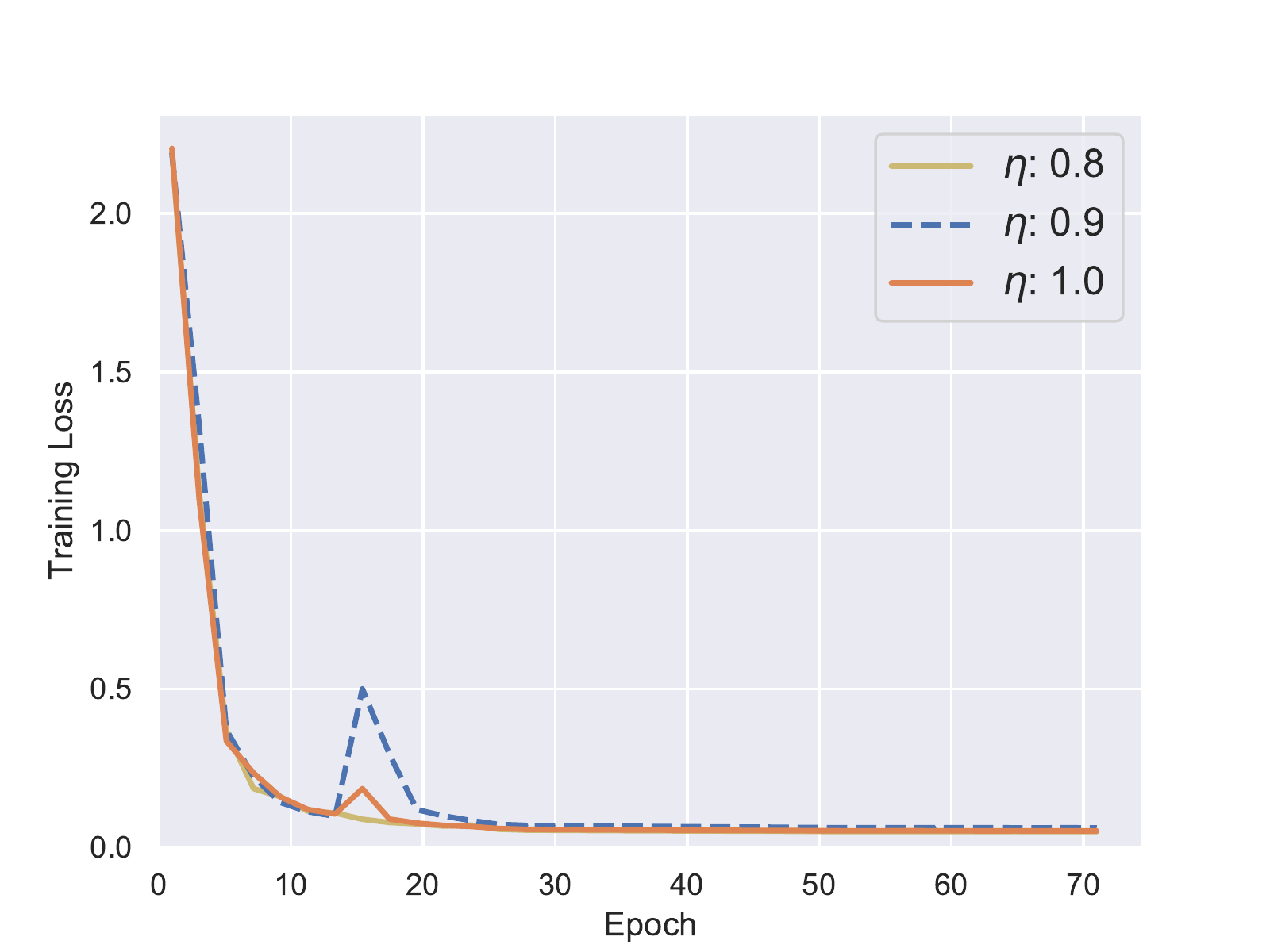}
\end{minipage}
}
\subfloat[Test loss]{
\begin{minipage}[t]{0.45\textwidth}
    \centering
    \includegraphics[width=0.9\textwidth]{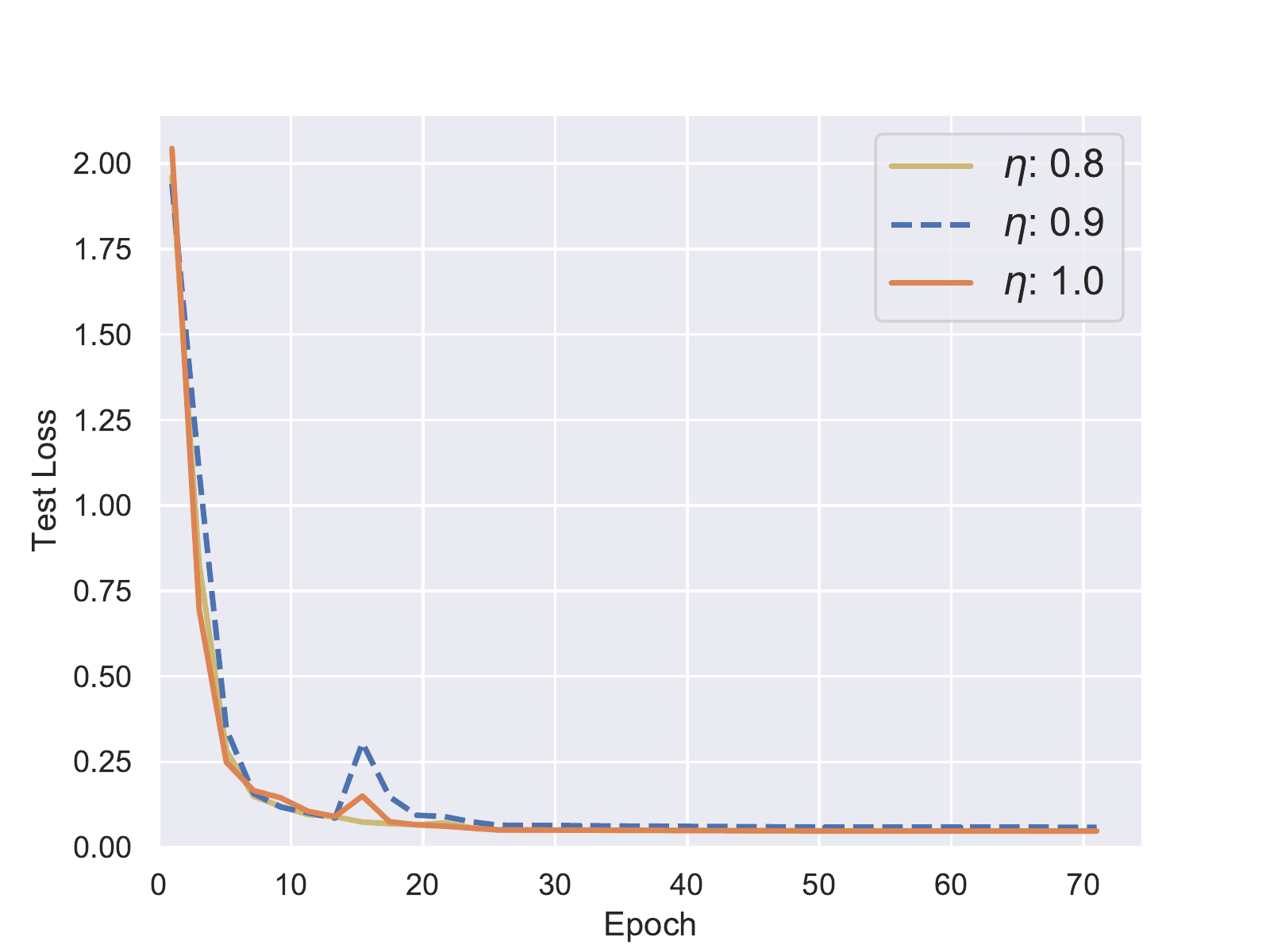}
\end{minipage}
}
\caption{Training under different $\eta$} 
\label{fig:eta}
\end{figure*}

\newpage
~
\newpage
\subsection{DORE in the smooth case}\label{sec:smoothcase}
\begin{algorithm}[!ht]
\caption{DORE with $R(\vx)=0$}\label{algosmooth}
\begin{algorithmic}[1]
\STATE \textbf{Input:} Stepsize $\alpha, \beta, \gamma, \eta$, initialize $\vh^0 = \vh_i^0 = \vzero^{d}$, $\hat{\vx}_i^0 = \hat{\vx}^0, ~\forall i\in\{1,\dots, n\}$. 
\FOR{$k=1,2,\cdots, K-1$}
\vspace{0.02in}
\begin{minipage}[t]{0.5\textwidth}
    \STATE \textbf{For each worker} $\{i=1,2,\cdots, n\}$:
    \vspace{0.05in}
    \STATE Sample $\vg_i^k$ such that $\EE [\vg_i^{k}|\hat{\vx}_i^{k}]=\nabla f_i(\hat{\vx}_i^k)$
    \STATE Gradient residual: $\Delta_i^k = \vg_i^k - \vh_i^k $
    \STATE Compression: $\hat{\Delta}_i^k = Q(\Delta_i^k)$
    \STATE $\vh_i^{k+1} = \vh_i^k + \alpha \hat{\Delta}_i^k$
`    \STATE $\{~\hat{\vg}_i^k = \vh_i^k + \hat{\Delta}_i^k~\}$
    \STATE Sent $\hat{\Delta}_i^k$ to the master
    \STATE Receive $\hat{\vq}^k$ from the master
    \STATE $\hat{\vx}_i^{k+1} = \hat{\vx}_i^k +\beta\hat{\vq}^k$
\end{minipage}
\begin{minipage}[t]{0.4\textwidth}
    \STATE \textbf{For the master}:
    \vspace{0.05in}
    \STATE Receive $\hat{\Delta}_{i}^{k}$s from workers
    \STATE $\hat{\Delta}^k = 1/n \sum_i^n {\hat{\Delta}_i^k}$
    \STATE $\hat{\vg}^k = \vh^k + \hat{\Delta}^k ~~\{=1/n \sum_i^n {\hat{\vg}_i^k}~\}$
    \STATE $\vh^{k+1} = \vh^k + \alpha \hat{\Delta}^k$ 
    \STATE $\vq^{k} = -\gamma {\hat\vg}^k + \eta \ve^k$ 
    \STATE Compression: $\hat{\vq}^k = Q(\vq^k)$
    \STATE $\ve^{k+1} = \vq^k - {\hat\vq}^k$ 
    \STATE Broadcast $\hat{\vq}^{k}$ to workers
\end{minipage}
\vspace{0.1in}
\ENDFOR
\STATE \textbf{Output:} any $\hat{\vx}_i^{K}$ \label{output}
\end{algorithmic}
\end{algorithm}

\subsection{Proof of Theorem~\ref{thm1}}
We first provide two lemmas. 
We define $\EE_Q$, $\EE_k$, and $\EE$ be the expectation taken over the quantization, the $k$th iteration based on $\hat\vx^k$, and the overall expectation, respectively.
\begin{lemma}\label{lem3}
For every $i$, we can estimate the first two moments of $\vh_{i}^{k+1}$ as
    \begin{align}\label{eq3}
    \EE_{Q}\vh_{i}^{k+1}=& (1-\alpha)\vh_{i}^{k}+\alpha \vg_{i}^{k},\\
    \EE_{Q}\|\vh_{i}^{k+1}-\vs_{i}\|^{2}\leq& (1-\alpha)\|\vh_{i}^{k}-\vs_{i}\|^{2}+\alpha\|\vg_{i}^{k}-\vs_{i}\|^{2}
    +\alpha[(C_{q}+1)\alpha-1]\|\Delta_{i}^{k}\|^{2}. \label{lemma2b}
    \end{align}
\end{lemma}

\begin{proof}
The first equality follows from lines~\ref{ite1}-\ref{ite3} of Algorithm~\ref{algocomplete} and Assumption~\ref{ass:compression}. 
For the second equation, we have the following variance decomposition 
\begin{equation}\label{var-decom}
    \EE\|X\|^{2}=\|\EE X\|^{2}+\EE\|X-\EE X\|^{2}
\end{equation}
for any random vector $X$.
By taking $X=\vh^{k+1}_{i}-\vs_i$, we get
\begin{equation}
    \EE_{Q}\|\vh_{i}^{k+1}-\vs_{i}\|^{2}=\|(1-\alpha)(\vh^{k}_{i}-\vs_{i})+\alpha(\vg^{k}_{i}-\vs_{i})\|^{2}+\alpha^{2}\EE_{Q}\|\hat{\Delta}^{k}_{i}-\Delta_{i}^{k}\|^{2}.
\end{equation}
Using the basic equality
\begin{equation}\label{basic_equality}
\|\lambda \va+(1-\lambda)\vb\|^{2}+\lambda(1-\lambda)\|\va-\vb\|^{2}=\lambda\|\va\|^{2}+(1-\lambda)\|\vb\|^{2}
\end{equation}
for all $\va,\vb\in\RR^{d}$ and $\lambda\in[0,1]$, as well as  Assumption~\ref{ass:compression}, we have 
\begin{align}
    \EE_{Q}\|\vh_{i}^{k+1}-\vs_{i}\|^{2}
    \leq(1-\alpha)\|\vh_{i}^{k}-\vs_{i}\|^{2}+\alpha\|\vg_{i}^{k}-\vs_{i}\|^{2}-\alpha(1-\alpha)\|\Delta_{i}^{k}\|^{2}+\alpha^{2}C_{q}\|\Delta_{i}^{k}\|^{2},
\end{align}
which is the inequality~\eqref{lemma2b}.
\end{proof}

Next, from the variance decomposition~\eqref{var-decom}, we also derive Lemma~\ref{lem4}.

\begin{lemma}\label{lem4}
The following inequality holds
\begin{align}\label{eq4}
    \EE[\|\hat{\vg}^{k}-\vh^{*}\|^{2}]\leq\EE\|\nabla f(\hat\vx^{k})-\vh^{*}\|^{2}+\frac{C_{q}}{n^{2}}\sum_{i=1}^{n}\EE\|\Delta_{i}^{k}\|^{2}+\frac{\sigma^{2}}{n},  
\end{align}
where $\vh^{*}=\nabla f(\vx^*)=\frac{1}{n}\sum_{i=1}^{n}\vh_{i}^{*} $ and $\sigma^{2}=\frac{1}{n}\sum_{i=1}^{n}\sigma_{i}^{2}$. 
\end{lemma}
\begin{proof}
By taking the expectation over the quantization of $\vg$, we have
\begin{align}\label{eq4_1}
    \EE\|\hat{\vg}^{k}-\vh^{*}\|^{2}&=\EE\|\vg^{k}-\vh^{*}\|^{2}+\EE\|\hat{\vg}^{k}-\vg^{k}\|^{2}\nonumber\\
    &\leq \EE\|\vg^{k}-\vh^{*}\|^{2}+\frac{C_{q}}{n^{2}}\sum_{i=1}^{n}\EE\|\Delta^{k}_{i}\|^{2},
\end{align}
where the inequality is from Assumption~\ref{ass:compression}.

For $\|\vg^{k}-\vh^{*}\|,$ we take the expectation over the sampling of gradients and derive
\begin{align}\label{eq4_2}
    \EE\|\vg^{k}-\vh^{*}\|^{2}&=\EE\|\nabla f(\hat{\vx}^{k})-\vh^{*}\|^{2}+\EE\|\vg^{k}-\nabla f(\hat\vx^{k})\|^{2}\nonumber\\
    &\leq \EE\|\nabla f(\hat\vx^{k})-\vh^{*}\|^{2}+\frac{\sigma^{2}}{n}
\end{align} 
by Assumption~\ref{asm1}. 

Combining~\eqref{eq4_1} with~\eqref{eq4_2} gives~\eqref{eq4}.
\end{proof}

\begin{proof}[Proof of Theorem~\ref{thm1}]
We consider $\vx^{k+1}-\vx^{*}$ first. Since $\vx^{*}$ is the solution of~\eqref{pb1}, it satisfies
\begin{equation}
    \vx^{*}=\prox_{\gamma R}(\vx^{*}-\gamma \vh^{*}).
\end{equation}
Hence
\begin{align}\label{eq5}
    \EE\|\vx^{k+1}-\vx^{*}\|^{2}= &\EE\|\prox_{\gamma R}(\hat{\vx}^{k}-\gamma\hat{\vg}^{k})-\prox_{\gamma R}(\vx^{*}-\gamma \vh^{*})\|^{2}\nonumber\\
    \leq& \EE\|\hat{\vx}^{k}-\vx^{*}-\gamma(\hat{\vg}^{k}-\vh^{*})\|^{2}\nonumber\\
    =& \EE\|\hat{\vx}^{k}-\vx^{*}\|^{2}-2\gamma\EE\dotp{\hat{\vx}^{k}-\vx^{*}, \hat{\vg}^{k}-\vh^{*}}+\gamma^{2}\EE\|\hat{\vg}^{k}-\vh^{*}\|^{2}\nonumber\\
    =& \EE\|\hat{\vx}^{k}-\vx^{*}\|^{2}-2\gamma\EE\dotp{\hat{\vx}^{k}-\vx^{*}, \nabla f(\hat{\vx}^{k})-\vh^{*}}+\gamma^{2}\EE\|\hat{\vg}^{k}-\vh^{*}\|^{2},
\end{align}
where the inequality comes from the non-expansiveness of the proximal operator and the last equality is derived by taking the expectation of the stochastic gradient $\hat{\vg}^{k}$.
Combining~\eqref{eq4} and~\eqref{eq5}, we have
\begin{align}\label{eq7}
    \EE\|\vx^{k+1}-\vx^{*}\|^{2}\leq&\EE\|\hat{\vx}^{k}-\vx^{*}\|^{2}-2\gamma\EE\dotp{\hat{\vx}^{k}-\vx^{*}, \nabla f(\hat{\vx}^{k})-\vh^{*}}\nonumber\\
    &+\frac{\gamma^{2}}{n}\sum_{i=1}^{n}\EE\|\nabla f_{i}(\hat{\vx}^{k})-\vh_{i}^{*}\|^{2}+\frac{C_{q}\gamma^{2}}{n^{2}}\sum_{i=1}^{n}\EE\|\Delta_{i}^{k}\|^{2}+\frac{\gamma^{2}}{n}\sigma^{2}.
\end{align}

Then we consider $\EE\|\hat{\vx}^{k+1}-\vx^{*}\|^{2}$. According to Algorithm~\ref{algocomplete}, we have:
\begin{align}
\EE_Q[\hat{\vx}^{k+1}-\vx^{*}] &= \hat{\vx}^{k}+\beta \vq^k-\vx^{*}\nonumber\\
&=  (1-\beta)(\hat{\vx}^{k}-\vx^*)+\beta(\vx^{k+1}-\vx^*+\eta\ve^k)
\end{align}
where the expectation is taken on the quantization of $\vq^k.$

By variance decomposition~\eqref{var-decom} and the basic equality~\eqref{basic_equality},
\begin{align}
  &\EE\|\hat{\vx}^{k+1}-\vx^{*}\|^{2}\nonumber\\
  \leq&(1-\beta)\EE\|\hat{\vx}^{k}-\vx^{*}\|^{2}+\beta\EE\|\vx^{k+1}+\eta\ve^k-\vx^{*}\|^{2}-\beta(1-\beta)\EE\|\vq^{k}\|^{2}+\beta^{2}C_{q}^m\EE\|\vq^{k}\|^{2}\nonumber\\
  \leq&(1-\beta)\EE\|\hat{\vx}^{k}-\vx^{*}\|^{2}+(1+\eta^2\epsilon)\beta\EE\|\vx^{k+1}-\vx^{*}\|^{2}-\beta(1-(C_q^m+1)\beta)\EE\|\vq^k\|^2 \nonumber \\
  &+(\eta^2+\frac{1}{\epsilon})\beta C_q^m\EE\|\vq^{k-1}\|^2,
  \end{align}
where $\epsilon$ is generated from Cauchy inequality of inner product. 
For convenience, we let $\epsilon=\frac{1}{\eta}$.

Choose a $\beta$ such that $0<\beta\leq \frac{1}{1+C_q^m}$.
Then we have 
\begin{align}
        &\beta(1-(C_q^m+1)\beta)\EE\|\vq^k\|^2+\EE\|\hat{\vx}^{k+1}-\vx^{*}\|^{2} \nonumber\\
\leq    & (1-\beta)\EE\|\hat{\vx}^{k}-\vx^{*}\|^{2}+(1+ \eta)\beta    \EE\|\vx^{k+1}-\vx^{*}\|^{2}+(\eta^2+\eta)\beta C_q^m\EE\|\vq^{k-1}\|^2.
\end{align}
  
Letting $\vs_i=\vh_i^*$ in~\eqref{lemma2b}, we have
\begin{align}
        &\frac{(1+\eta)c\beta\gamma^{2}}{n}\sum_{i=1}^{n} \EE\|\vh^{k+1}_{i}-\vh_{i}^{*}\|^{2} \nonumber\\
\leq    & \frac{(1+\eta)(1-\alpha)c\beta\gamma^{2}}{n}\sum_{i=1}^{n}\|\vh_{i}^{k}-\vh^{*}_{i}\|^{2}+\frac{(1+\eta)\alpha c\beta\gamma^2}{n}\sum_{i=1}^{n}\|\vg_{i}^{k}-\vh^*_{i}\|^{2}\nonumber\\
 &+\frac{(1+\eta)\alpha[(C_{q}+1)\alpha-1]c\beta\gamma^2}{n}\sum_{i=1}^{n}\|\Delta_{i}^{k}\|^{2}.
\end{align}

Then we let $\vR^k=\beta(1-(C_q^m+1)\beta)\EE\|\vq^k\|^2$ and define $\vV^k=\vR^{k-1}+\EE\|\hat{\vx}^{k}-\vx^{*}\|^{2}+\frac{(1+\eta)c\beta\gamma^{2}}{n}\sum_{i=1}^{n}\EE\|\vh_{i}^{k}-\vh_{i}^{*}\|^{2}$. 
Thus, we obtain
\begin{align}\label{eq10}
  \vV^{k+1} 
  \leq&(\eta^2+\eta)\beta C_q^m\EE\|\vq^{k-1}\|^2+(1+\eta\beta) \EE\|\hat{\vx}^{k}-\vx^{*}\|^{2}-2(1+\eta)\beta\gamma\EE\dotp{\hat{\vx}^{k}-\vx^{*}, \nabla f(\hat{\vx}^{k})-\vh^{*}}\nonumber\\
  &+\frac{(1+\eta)(1-\alpha)c\beta\gamma^{2}}{n}\sum_{i=1}^{n}\EE\|\vh_{i}^{k}-\vh_{i}^{*}\|^{2}+\frac{(1+\eta)\beta\gamma^{2}}{n^{2}}\Big[nc (C_{q}+1)\alpha^{2}-nc\alpha+C_{q}\Big]\sum_{i=1}^{n}\EE\|\Delta_{i}^{k}\|^{2}\nonumber\\
  &+\frac{(1+\eta)(1+c\alpha)}{n}\beta\gamma^{2}\sum_{i=1}^{n}\EE\|\nabla f_{i}(\hat{\vx}^{k})-\vh_{i}^{*}\|^{2}+\frac{(1+\eta)(1+nc\alpha)}{n}\beta\gamma^{2}\sigma^{2}.
 \end{align}

The $\EE\|\Delta_{i}^{k}\|^{2}$-term can be ignored if $nc(C_{q}+1)\alpha^{2}-nc\alpha+C_{q}\leq0$, which can be guaranteed by $c\geq\frac{4C_{q}(C_{q}+1)}{n}$ and 
$$\alpha\in \left({1- \sqrt{1-{4C_q(C_q+1)\over nc}}\over 2(C_q+1)} , {1+ \sqrt{1-{4C_q(C_q+1)\over nc}}\over 2(C_q+1)}\right).$$


Given that each $f_{i}$ is $L$-Lipschitz differentiable and $\mu$-strongly convex, we have 
\begin{align}\label{eq11}
    \EE\dotp{\nabla f(\hat{\vx}^{k})-\vh^{*}, \hat{\vx}^{k}-\vx^{*}}
\geq\frac{\mu L}{\mu+L}\EE\|\hat{\vx}^{k}-\vx^{*}\|^{2}+\frac{1}{\mu+L}\frac{1}{n}\sum_{i=1}^{n}\EE\|\nabla f_{i}(\hat{\vx}^{k})-\vh_{i}^{*}\|^{2}.
\end{align}
Hence
  \begin{align}
  \vV^{k+1}\leq& \rho_1\vR^{k-1}+(1+\eta\beta)\EE\|\hat{\vx}^{k}-\vx^{*}\|^{2}-2(1+\eta)\beta\gamma\EE\dotp{\hat{\vx}^{k}-\vx^{*}, \nabla f(\hat{\vx}^{k})-\vh^{*}}\nonumber\\
  &+\frac{(1+\eta)(1-\alpha)c\beta\gamma^{2}}{n}\sum_{i=1}^{n}\EE\|\vh_{i}^{k}-\vh_{i}^{*}\|^{2}+\frac{(1+\eta)(1+c\alpha)}{n}\beta\gamma^{2}\sum_{i=1}^{n}\EE\|\nabla f_{i}(\hat{\vx}^{k})-\vh_{i}^{*}\|^{2} \nonumber \\
  &+\frac{(1+\eta)(1+nc\alpha)}{n}\beta\gamma^{2}\sigma^{2}\nonumber\\
  \leq& \rho_1\vR^{k-1}+\Big[ 1+\eta\beta - \frac{2(1+\eta)\beta\gamma \mu L}{\mu +L}\Big] \EE\|\hat{\vx}^{k}-\vx^{*}\|^{2}+\frac{(1+\eta)(1-\alpha)c\beta\gamma^{2}}{n}\sum_{i=1}^{n}\EE\|\vh_{i}^{k}-\vh_{i}^{*}\|^{2}\nonumber \\
  &+\Big[(1+\eta)(1+c\alpha)\beta\gamma^{2} - \frac{2(1+\eta)\beta \gamma}{\mu+L}\Big]\frac{1}{n} \sum_{i=1}^{n}\EE\|\nabla f_{i}(\hat{\vx}^{k})-\vh_{i}^{*}\|^{2}+\frac{(1+\eta)(1+nc\alpha)}{n}\beta\gamma^{2}\sigma^{2}\nonumber\\
  \leq& \rho_1\vR^{k-1}+\rho_2\EE\|\hat{\vx}^{k}-\vx^{*}\|^{2}+\frac{(1+\eta)(1-\alpha)c\beta\gamma^{2}}{n}\sum_{i=1}^{n}\EE\|\vh_{i}^{k}-\vh_{i}^{*}\|^{2}+\frac{(1+\eta)(1+nc\alpha)}{n}\beta\gamma^{2}\sigma^{2}
  \end{align}
where 
\begin{align*}
\rho_1 = & \frac{(\eta^2+\eta)C_q^m}{1-(C_q^m+1)\beta}, \\
\rho_2 = & 1+\eta\beta-\frac{2(1+\eta)\beta\gamma\mu L}{\mu+L}.
\end{align*}
Here we let $\gamma \leq \frac{2}{(1+c\alpha)(\mu+L)}$ such that $(1+\eta)(1+c\alpha)\beta\gamma^{2} - \frac{2(1+\eta)\beta \gamma}{\mu+L} \leq 0$ and the last inequality holds. In order to get $\max(\rho_1, \rho_2, 1-\alpha) < 1$, we have the following conditions
\begin{align*}
    0 \leq (\eta^2+\eta)C_q^m  \leq  & 1-(C_q^m+1)\beta,  \\
    \eta <  &\frac{2(1+\eta)\gamma\mu L}{\mu+L}.
\end{align*}
Therefore, the condition for $\gamma$ is 
$${\eta(\mu+L)\over 2(1+\eta)\mu L}\leq \gamma\leq {2\over (1+c\alpha)(\mu+L)},$$
which implies an additional condition for $\eta$. 
Therefore, the condition for $\eta$ is 
$$\eta\in \left[0,\min\left({-C_q^m+ \sqrt{(C_q^m)^2+4(1-(C_q^m+1)\beta)}\over 2C_q^m}, {4\mu L\over(\mu+L)^2(1+c\alpha)-4\mu L}\right)\right).$$

where $\eta \leq {4\mu L\over(\mu+L)^2(1+c\alpha)-4\mu L}$ is to ensure
${\eta(\mu+L)\over 2(1+\eta)\mu L}\leq {2\over (1+c\alpha)(\mu+L)}$ such that we don't get an empty set for $\gamma$.

If we define $\rho = \max \{\rho_1, \rho_2, 1-\alpha \},$ we obtain
 \begin{equation}\label{ineq_V}
     \vV^{k+1}\leq \rho\vV^k+\frac{(1+\eta)(1+nc\alpha)}{n}\beta\gamma^2\sigma^2
 \end{equation}
 and the proof is completed by applying~\eqref{ineq_V} recurrently.
\end{proof}

\subsection{Proof of Theorem~\ref{thm3}}
\begin{proof}In Algorithm~\ref{algosmooth}, we can show 
\begin{equation}\label{equ:x_k}
    \begin{aligned}
    \EE\|\hat{\vx}^{k+1}-\hat{\vx}^{k}\|^{2}&=\beta^2\EE\|\hat{\vq}^k\|^2
    =\beta^2\EE\|\EE\hat{\vq}^k\|^2+\beta^2\EE\|\hat{\vq}^k-\EE\hat{\vq}^k\|^2\\
    &=\beta^2\EE\|{\vq}^k\|^2+\beta^2\EE\|\hat{\vq}^k-{\vq}^k\|^2\\
    &\leq (1+C_q^m)\beta^2\EE\|{\vq}^k\|^2.
    \end{aligned}
\end{equation}
and 
\begin{equation}\label{equ:q_k}
 \EE\|\vq^{k}\|^2=\EE\|-\gamma\hat{\vg}^{k}+\eta\ve^{k}\|^2\leq
 2\gamma^2\EE\|\hat{\vg}^{k}\|^2+2\eta^2\EE\|\ve^k\|^2 \leq \\
 2\gamma^2\EE\|\hat{\vg}^{k}\|^2+2C_q^m\eta^2\EE\|\vq^{k-1}\|^2. 
\end{equation}

Using~(\ref{equ:x_k})(\ref{equ:q_k}) and the Lipschitz continuity of $\nabla f(\vx)$, we have
\begin{align}\label{equ:lip}
  &\EE f(\hat{\vx}^{k+1}) + (C_q^m+1)L\beta^2\EE\|\vq^{k}\|^2 \nonumber \\
  \leq& \EE f(\hat{\vx}^{k})+\EE\dotp{\nabla f(\hat{\vx}^{k}), \hat{\vx}^{k+1}-\hat{\vx}^{k}}+\frac{L}{2}\EE\|\hat{\vx}^{k+1}-\hat{\vx}^{k}\|^{2} + (C_q^m+1)L\beta^2\EE\|\vq^{k}\|^2 \nonumber \\
  = &\EE f(\hat{\vx}^{k})+\beta \EE\dotp{\nabla f(\hat{\vx}^{k}), -\gamma\hat{\vg}^{k}+\eta \ve^{k}}+\frac{(1+C_q^m)L\beta^2}{2}\EE\|{\vq}^k\|^2 + (C_q^m+1)L\beta^2\EE\|\vq^{k}\|^2 \nonumber\\
  =& \EE f(\hat{\vx}^{k})+\beta \EE\dotp{\nabla f(\hat{\vx}^{k}), -\gamma\nabla f(\hat{\vx}^{k})+\eta \ve^{k}} + \frac{3(C_q^m+1)L\beta^2}{2}\EE\|\vq^{k}\|^2 \nonumber\\
  \leq&\EE f(\hat{\vx}^{k})-\beta\gamma\EE\|\nabla f(\hat{\vx}^{k})\|^{2}+\frac{\beta\eta}{2}\EE\|\nabla f(\hat{\vx}^{k})\|^2 + \frac{\beta\eta}{2} \EE\|\ve^{k}\|^2 \nonumber\\
  & ~~~~+ 3(C_q^m+1)L\beta^2 \Big[\gamma^2\EE\|\hat{\vg}^{k}\|^2+C_q^m\eta^2\EE\|\vq^{k-1}\|^2\Big]  \nonumber\\
  \leq&\EE f(\hat{\vx}^{k})-\Big[\beta\gamma-\frac{\beta\eta}{2}-3(C_{q}^{m}+1)L\beta^{2}\gamma^2\Big]\EE\|\nabla f(\hat{\vx}^{k})\|^{2}\nonumber\\
  &+\frac{3C_{q}(C^{m}_{q}+1)L\beta^{2}\gamma^2}{n^{2}}\sum_{i=1}^{n}\EE\|\Delta^{k}_{i}\|^2+\frac{3(C_{q}^m+1)L\beta^{2}\gamma^2}{n}\sigma^{2}\nonumber\\
  &+\Big[\frac{\beta\eta C_q^m}{2}+(3C_{q}^m+1)C_q^m L\beta^{2}\eta^2 \Big] \EE\|\vq^{k-1}\|^{2},
\end{align}
where the last inequality is from~\eqref{eq4} with $\vh^{*}=\vzero$.

Letting $\vs_i=\vzero$ in~\eqref{lemma2b}, we have
\begin{align}\label{noncon_eq1}
    \EE_{Q}\|\vh_{i}^{k+1}\|^{2}\leq& (1-\alpha)\|\vh_{i}^{k}\|^{2}+\alpha\|\vg_{i}^{k}\|^{2}+\alpha[(C_{q}+1)\alpha-1]\|\Delta_{i}^{k}\|^{2}.
\end{align}

Due to the assumption that each worker samples the gradient from the full dataset, we have \begin{equation}\label{noncon_eq2}
    \EE\vg^{k}_{i}=\EE\nabla f(\hat{\vx}^{k}),\quad \EE\|\vg_i^{k}\|^{2}\leq\EE\|\nabla f(\hat{\vx}^{k})\|^{2}+\sigma^{2}_{i}.
\end{equation}

Define $\Lambda^{k}=(C_q^m+1)L\beta^2\|\vq^{k-1}\|^2+f(\hat{\vx}^{k})-f^{*}+3c(C_{q}^{m}+1)L\beta^{2}\gamma^2\frac{1}{n}\sum_{i=1}^{n}\EE\|\vh^{k}_{i}\|^{2},$ and from~\eqref{equ:lip}, \eqref{noncon_eq1}, and~\eqref{noncon_eq2}, we have
\begin{align}
  \EE\Lambda^{k+1}
  \leq&\EE f(\hat{\vx}^{k})-f^{*}+3(1-\alpha)c(C_{q}^{m}+1)L\beta^{2}\gamma^2\frac{1}{n}\sum_{i=1}^{n}\EE\|\vh^{k}_{i}\|^{2}\nonumber\\
  &-\Big[\beta\gamma-\frac{\beta\eta}{2}-3(1+c\alpha)(C_{q}^{m}+1)L\beta^{2}\gamma^2\Big]\EE\|\nabla f(\hat{\vx}^{k})\|^{2}\nonumber\\
  &+\frac{(C_{q}^{m}+1)L\beta^{2}\gamma^2}{n^{2}}\Big[3nc(C_{q}+1)\alpha^{2}-3nc\alpha+3C_{q}\Big]\sum_{i=1}^{n}\EE\|\Delta_i^k\|^2\nonumber\\
  &+3(1+nc\alpha)\frac{(C_{q}^m+1)L\beta^{2}\gamma^2\sigma^{2}}{n}\nonumber\\
  &+\Big[\frac{\beta\eta C_q^m}{2} + 3(C_{q}^m+1)C_q^m L\beta^{2}\eta^2 \Big] \EE\|\vq^{k-1}\|^{2}.
\end{align}

If we let $c=\frac{4C_q(C_q+1)}{n}$, then the condition of $\alpha$ in~\eqref{alpha-beta} gives $3nc(C_{q}+1)\alpha^{2}-3nc\alpha+3C_{q}\leq0$ and
\begin{align}
  \EE\Lambda^{k+1}
  \leq&\EE f(\hat{\vx}^{k})-f^{*}+3(1-\alpha)c(C_{q}^{m}+1)L\beta^2\gamma^2\frac{1}{n}\sum_{i=1}^{n}\EE\|\vh^{k}_{i}\|^{2}\nonumber  \\ 
  &-\Big[\beta\gamma-\frac{\beta\eta}{2}-3(1+c\alpha)(C_{q}^{m}+1)L\beta^{2}\gamma^2\Big]\EE\|\nabla f(\hat{\vx}^{k})\|^{2}\nonumber\\
  &+3(1+nc\alpha)\frac{(C_{q}^m+1)L\beta^{2}\gamma^2\sigma^{2}}{n}\nonumber\\
  &+[\frac{\beta\eta C_q^m}{2}+3(C_{q}^m+1)C_q^m L\beta^{2}\eta^2]\EE\|\vq^{k-1}\|^{2}.
\end{align}

Let $\eta=\gamma$ and $\beta\gamma \leq \frac{1}{6(1+c\alpha)(C_q^m+1)L}$, we have
\begin{equation*}
    \beta\gamma-\frac{\beta\eta}{2}-3(1+c\alpha)(C_{q}^{m}+1)L\beta^{2}\gamma^2 = 
    \frac{\beta\gamma}{2}-3(1+c\alpha)(C_{q}^{m}+1)L\beta^{2}\gamma^2 \geq 0.
\end{equation*}

Take $\gamma \leq \min\Big\{\frac{-1+\sqrt{1+ \frac{48L^2\beta^2(C_q^m+1)^2}{C_q^m}}}{12L\beta(C_q^m+1)}, \frac{1}{6L\beta(1+c\alpha)(C_q^m+1)}\Big\}$ will guarantee
\begin{equation*}
\Big[\frac{\beta\eta C_q^m}{2}+3(C_{q}^m+1)C_q^m L\beta^{2}\eta^2\Big]\leq(C_q^m+1)L\beta^2.
\end{equation*}

Hence we obtain
\begin{equation}\label{eq14}
    \EE\Lambda^{k+1}\leq\EE\Lambda^{k}-\Big[\frac{\beta\gamma}{2}-3(1+c\alpha)(C_{q}^{m}+1)L\beta^{2}\gamma^2\Big]\EE\|\nabla f(\hat{\vx}^{k})\|^{2}+3(1+nc\alpha)\frac{(C_{q}^m+1)L\beta^{2}\gamma^2\sigma^{2}}{n}.
\end{equation}

Taking the telescoping sum and plugging the initial conditions, we derive~\eqref{nonconvex-conv}.
\end{proof}

\subsection{Proof of Corollary~\ref{col2}}
\begin{proof}

With $\alpha=\frac{1}{2(C_{q}+1)}$ and $c=\frac{4C_{q}(C_{q}+1)}{n}$, $1+nc\alpha=1+2C_{q}$ is a constant.
We set $\beta=\frac{1}{C^{m}_{q}+1}$ and $\gamma=\min\Big\{ \frac{-1+\sqrt{1+\frac{48L^2}{ C_q^m}}}{12L}, \frac{1}{12L(1+c\alpha)(1+\sqrt{K/n})} \Big\}$. In general, $C_q^m$ is bounded which makes the first bound negligible, i.e.,  $\gamma = \frac{1}{12L(1+c\alpha)(1+\sqrt{K/n})}$ when $K$ is large enough. Therefore, we have
\begin{equation}
\frac{\beta}{2}-3(1+c\alpha)(C_{q}^{m}+1)L\beta^{2}\gamma
= \frac{1-6(1+c\alpha)L\gamma}{2(C_q^m+1)} \leq \frac{1}{4(C_q^m+1)}.
\end{equation}

From Theorem~\ref{thm3}, we derive
\begin{align}
  &\frac{1}{K}\sum_{k=1}^{K} \EE\|\nabla f(\hat{\vx}^{k})\|^{2} \nonumber\\
  \leq& \frac{4(C_q^m+1)(\EE\Lambda^{1}-\EE\Lambda^{K+1})}{\gamma K} + \frac{12(1+nc\alpha)L\sigma^2\gamma}{n} \nonumber\\
  \leq& 48L(C_q^m+1)(1+c\alpha)(\EE\Lambda^{1}-\EE\Lambda^{K+1})({1\over K}+{1\over \sqrt{nK}}) + \frac{(1+nc\alpha)\sigma^{2}}{(1+c\alpha)}\frac{1}{\sqrt{nK}},
\end{align}
which completes the proof.

\end{proof}

\end{document}